\documentclass[twoside,11pt]{article}

\usepackage[preprint]{jmlr2e}

\ShortHeadings{Fast and Accurate LMS Solvers}{Maalouf and Jubran and Feldman}
\firstpageno{1}

\usepackage{enumerate}
\usepackage{amsmath,amssymb, latexsym,graphicx}
\usepackage{url}
\usepackage{algorithm}
\usepackage{color}
\usepackage{float}
\usepackage{amssymb}
\usepackage{mathtools}
\usepackage{graphicx}
\usepackage[algo2e,ruled,noend,linesnumbered]{algorithm2e}
\usepackage{subcaption}
\usepackage[font={small,it}]{caption}
\usepackage{makecell}
\usepackage{adjustbox}
\usepackage{adjustbox}
\usepackage{placeins}
\usepackage{natbib}

\newtheorem{observation}[theorem]{Observation}

\newcommand{\br}[1]{\left\{#1\right\}}

\usepackage{amsmath}

\newcommand{\comment}[1]{}
\newcommand{\REAL}{\ensuremath{\mathbb{R}}}

\newcommand{\ceil}[1]{\left \lceil #1 \right \rceil}

\newcommand{\smi}{\sum_{i=1}^n}

\newcommand{\norm}[1]{\left\lVert#1\right\rVert}
\newcommand{\eps}{\varepsilon}

\newcommand{\RidgeOurAlg}{\textsc{Ridgecv-Boost}}

\newcommand{\LassoOurAlg}{\textsc{Lassocv-Boost}}

\newcommand{\calcCoreset}{\textsc{LMS-Coreset}}
\newcommand{\LstsqOurAlg}{\textsc{LinReg-Boost}}

\newcommand{\folds}{m}

\newcommand{\ElasticOurAlg}{\textsc{Elasticcv-Boost}}

\newcommand{\caracord}{\textsc{Sparse-Caratheodory-Set}}
\newcommand{\caraf}{\textsc{Fast-Caratheodory-Set}}
\newcommand{\cova}{\textsc{Caratheodory-Matrix}}
\newcommand{\cordCova}{\textsc{Sparse-Caratheodory-Matrix}}
\newcommand{\tcara}{t}
\newcommand{\caras}{\textsc{Caratheodory}}\makeatletter
\newcommand{\nosemic}{\renewcommand{\@endalgocfline}{\relax}}
\newcommand{\dosemic}{\renewcommand{\@endalgocfline}{\algocf@endline}}
\let\oldnl\nl
\newcommand{\nonl}{\renewcommand{\nl}{\let\nl\oldnl}}

\newcommand{\lstsqq}{\texttt{LinearRegression}}
\newcommand{\ridge}{\texttt{RidgeCV}}
\newcommand{\lasso}{\texttt{LassoCV}}
\newcommand{\elastic}{\texttt{ElasticNetCV}}
\newcommand{\sklearn}{\texttt{sklearn.linear\_model}}
\newcommand{\scipy}{\texttt{scipy.linalg}}
\newcommand{\alphas}{\mathbb{A}}

\usepackage{multirow,booktabs}

\newif\iftree
\treefalse

\newif\ifproofs
\proofsfalse

\newif\ifqr
\qrfalse

\begin{document}

\title{Fast and Accurate Least-Mean-Squares Solvers}

\author{\name Alaa Maalouf$\dagger$ \email                       
       alaamalouf12@gmail.com \\
       \name Ibrahim Jubran$\dagger$ \email 
       ibrahim.jub@gmail.com\\
        \name Dan Feldman \email 
       dannyf.post@gmail.com\\
       \addr \addr Robotics \& Big Data Labs\\
       Department of Computer Science\\
       University of Haifa\\
       Abba Khoushy Ave 199, Israel}

\thanks{$\dagger$ Those authors contributed equally to this work.\\ 
An extended abstract of this work was previously published at the Neural Information Processing Systems 2019 (NeurIPS'19)~\cite{maalouf2019fast}. In this extended version we provide faster algorithms which can support high dimensional data, extend our results to handle a wider range of problems, and conduct extensive new experimental results.}

\maketitle

\begin{abstract}
Least-mean squares (LMS) solvers such as Linear / Ridge / Lasso-Regression, SVD and Elastic-Net not only solve fundamental machine learning problems, but are also the building blocks in a variety of other methods, such as decision trees and matrix factorizations.

We suggest an algorithm that gets a finite set of $n$ $d$-dimensional real vectors and returns a weighted subset of $d+1$ vectors whose sum is \emph{exactly} the same. The proof in Caratheodory's Theorem (1907) computes such a subset in $O(n^2d^2)$ time and thus not used in practice. Our algorithm computes this subset in $O(nd+d^4\log{n})$ time, using $O(\log n)$ calls to Caratheodory's construction on small but ``smart'' subsets. This is based on a novel paradigm of fusion between different data summarization techniques, known as sketches and coresets.

For large values of $d$, we suggest a faster construction that takes $O(nd)$ time (linear in the input's size) and returns a weighted subset of $O(d)$ sparsified input points. Here, sparsified point means that some of its entries were replaced by zeroes.

As an example application, we show how it can be used to boost the performance of existing LMS solvers, such as those in scikit-learn library, up to x100. Generalization for streaming and distributed (big) data is trivial.
Extensive experimental results and complete open source code are also provided.
\end{abstract}


\section{Introduction and Motivation}
Least-Mean-Squares (LMS) solvers are the family of fundamental optimization problems in machine learning and statistics that include linear regression, Principle Component Analysis (PCA), Singular Value Decomposition (SVD), Lasso and Ridge regression, Elastic net, and many more~\cite{golub1971singular,jolliffe2011principal,hoerl1970ridge,seber2012linear,zou2005regularization,tibshirani1996regression,safavian1991survey}. See formal definition below. First closed form solutions for problems such as linear regression were published by e.g. Pearson~\cite{pearson1900x} around 1900 but were probably known before. Nevertheless, today they are still used extensively as building blocks in both academy and industry for normalization~\cite{liang2013distributed,kang2011scalable,afrabandpey2016regression}, spectral clustering~\cite{peng2015robust}, graph theory~\cite{zhang2018understanding}, prediction~\cite{copas1983regression,porco2015low}, dimensionality reduction~\cite{laparra2015dimensionality}, feature selection~\cite{gallagher2017cross} and many more; see more examples in~\cite{golub2012matrix}.

Least-Mean-Squares solver in this paper is an optimization problem that gets as input an $n\times d$ real matrix $A$, and another $n$-dimensional real vector $b$ (possibly the zero vector). It aims to minimize the sum of squared distances from the rows (points) of $A$ to some hyperplane that is represented by its normal or vector of $d$ coefficients $x$, that is constrained to be in a given set $X\subseteq\REAL^d$:
\begin{equation} \label{DefLMS}
\min_{x\in X}f(\norm{Ax-b}_2)+g(x).
\end{equation}
Here, $g$ is called a \emph{regularization term}. For example: in linear regression $X=\REAL^d$, $f(y)=y^2$ for every $y\in\REAL$ and $g(x)=0$ for every $x\in X$. In Lasso $f(y) = y^2$ for every $y\in\REAL$ and $g(x) = \alpha \cdot \norm{x}_1$ for every $x\in \REAL^d$ and $\alpha>0$.
Such LMS solvers can be computed via the covariance matrix $A^TA$. For example, the solution to linear regression of minimizing $\norm{Ax-b}_2$ is $(A^TA)^{-1}A^Tb$.

\subsection{Related work} \label{sec:relatedWork}
While there are many LMS solvers and corresponding implementations, there is always a trade-off between their accuracy and running time; see comparison table in~\cite{bauckhage2015numpy} with references therein. 
The reason is related to the fact that computing the covariance matrix of $A$ can be done essentially in one of two ways: (i) summing the $d\times d$ outer product $a_ia_i^T$ of the $i$th row $a_i^T$ of $A$ over every $i$, $1\leq i\leq n$. This is due to the fact that $A^TA=\sum_{i=1}^n a_ia_i^T$, or (ii) factorization of $A$, e.g. using SVD or the QR decomposition~\cite{golub1971singular}.

\noindent\textbf{Numerical issues. }Method (i) is easy to implement for streaming rows of $A$ by maintaining only $d^2$ entries of the covariance matrix for the $n$ vectors seen so far, or maintaining its inverse $(A^TA)^{-1}$ as explained e.g. in~\cite{golub2012matrix}. This takes $O(d^2)$ time for each vector insertion and requires $O(d^2)$ memory, which is the same as the desired output covariance matrix. However, every such addition may introduce another numerical error which accumulates over time. This error increases significantly when running the algorithms using 32 bit floating point representation, which is common for GPU computations; see Fig.~\ref{fig:accuracy} for example. This solution is similar to maintaining the set of $d$ rows of the matrix $DV^T$, where $A = UDV^T$ is the SVD of $A$, which is not a subset of the original input matrix $A$ but has the same covariance matrix $A^TA=VD^2V$.
A common problem is that to compute $(A^TA)^{-1}$, the matrix $A^TA$ must be invertible. This may not be the case due to numerical issues. In algorithms such as Lasso, the input cannot be a covariance matrix, but only a corresponding matrix whose covariance matrix is $A^TA$, that can be computed from the Cholesky decomposition~\cite{bjorck1967solving} that returns a left triangular matrix $A$ for the given covariance matrix $A^TA$. However, Cholesky decomposition can be applied only on positive-definite matrices, which is not the case even for small numerical errors that are added to $A^TA$. See Section~\ref{sec:ER} for more details and empirical evidence.

\noindent\textbf{Running-time issues. }Method (ii) above utilizes factorizations such as SVD, i.e., $A=UDV^T$ to compute the covariance matrix via $A^TA=VD^2V^T$ or the QR decomposition $A=QR$ to compute $A^TA=R^TQ^TQR^T=R^TR$. This approach is known to be much more stable. However, it is much more time consuming: while in theory the running time is $O(nd^2)$ as in the first method, the constants that are hidden in the $O(\cdot)$ notation are significantly larger.
Moreover, unlike Method (i), it is impossible to compute such factorizations exactly for streaming data~\cite{clarkson2009numerical}.

\noindent\textbf{Caratheodory's Theorem~\cite{caratheodory1907variabilitatsbereich} }states that every point contained in the convex hull of $n$ points in $\REAL^d$ can be represented as a convex combination of a subset of at most $d+1$ points, which we call the \emph{Caratheodory set}; see Section~\ref{sec:prelim} and Fig.~\ref{fig:booster}. This implies that we can maintain a weighted (scaled) set of $d^2+1$ points (rows) whose covariance matrix is the same as $A$, since $(1/n)\sum_i a_ia_i^T$ is the mean of $n$ matrices and thus in the convex hull of their corresponding points in $\REAL^{(d^2)}$; see Algorithm~\ref{covAlg}.
The fact that we can maintain such a small sized subset of points instead of updating linear combinations of all the $n$ points seen so far, significantly reduces the numerical errors as shown in Fig.~\ref{fig:accuracy}. Unfortunately, computing this set from Caratheodory's Theorem takes $O(n^2d^2)$ or $O(nd^3)$ time via $O(n)$ calls to an LMS solver. This fact makes it non-practical to use in an LMS solvers, as we aim to do in this work, and may explain the lack of software or source code for this algorithm on the web.

\noindent\textbf{Approximations via Coresets and Sketches. } In the recent decades numerous approximation and data summarization algorithms were suggested to \emph{approximate} the problem in~\eqref{DefLMS}; see e.g.~\cite{drineas2006sampling, jubran2019provable, clarkson2017low, maalouf2019tight} and references therein. One possible approach is to compute a small matrix $S$ whose covariance $S^TS$ approximates, in some sense, the covariance matrix $A^TA$ of the input data $A$. The term \emph{coreset} is usually used when $S$ is a weighted (scaled) subset of rows from the $n$ rows of the input matrix.
The matrix $S$ is sometimes called a \emph{sketch} if each rows in $S$ is a linear combination of few or all rows in $A$, i.e. $S=WA$ for some matrix $W\in\REAL^{s\times n}$.
However, those coresets and sketches usually yield $(1+\eps)$-multiplicative approximations for $\norm{Ax}_2^2$ by $\norm{Sx}_2^2$ where the matrix $S$ is of $(d/\eps)^{O(1)}$ rows and $x$ may be any vector, or the smallest/largest singular vector of $S$ or $A$; see lower bounds in~\cite{feldman2010coresets}.
Moreover, a $(1+\eps)$-approximation to $\norm{Ax}_2^2$ by $\norm{Sx}_2^2$ does not guarantee an approximation to the actual entries or eigenvectors of $A$ by $S$ that may be very different.

\noindent\textbf{Accurately handling big data. } The algorithms in this paper return \emph{accurate} coresets ($\eps=0$), which is less common in the literature; see~\cite{jubran2019introduction} for a brief summary. These algorithms can be used to compute the covariance matrix $A^TA$ via a scaled subset of rows from the input matrix $A$. Such coresets support unbounded stream of input rows using memory that is \emph{sub-linear} in their size, and also support dynamic/distributed data in parallel. This is by the useful merge-and-reduce property of coresets that allow them to handle big data; see details e.g. in~\cite{agarwal2004approximating}. Unlike traditional coresets that pay additional logarithmic multiplicative factors due to the usage of merge-reduce trees and increasing error, the suggested weighted subsets in this paper do not introduce additional error to the resulting compression since they preserve the desired statistics accurately. The actual numerical errors are measured in the experimental results, with analysis that explain the differences.

A main advantage of a coreset over a sketch is that it preserves sparsity of the input rows~\cite{feldmanmik}, which usually reduces theoretical running time. Our experiments show, as expected from the analysis, that coresets can also be used to significantly improve the numerical stability of existing algorithms. Another advantage is that the same coreset can be used for parameter tuning over a large set of candidates. In addition to other reasons, this significantly reduces the running time of such algorithms in our experiments; see Section~\ref{sec:ER}.

\subsection{Our contribution}
A natural question that follows from the previous section is: \emph{can we maintain the optimal solution for LMS problems both accurately and fast?} We answer this question affirmably by suggesting:
\begin{enumerate} [(i)]
\setlength{\itemsep}{-1pt}
\item  the first algorithm that computes the Caratheodory set of $n$ input points in $O(nd+d^4\log{n})$ time. This is by using a novel approach of coreset/skecthes fusion that is explained in the next section; see Algorithm~\ref{fastCara} and Theorem~\ref{fastCara}.

\item an algorithm that maintains a (``coreset'') matrix $S\in\REAL^{(d^2+1)\times d}$ such that: (a) its set of rows is a scaled subset of rows from $A\in\REAL^{n\times d}$ whose rows are the input points, and (b) the covariance matrices of $S$ and $A$ are the same, i.e.,  $S^TS=A^TA$; see Algorithm~\ref{covAlg} and Theorem~\ref{theorem:cov}.

\item a faster, yet potentially less numerically accurate, algorithm for computing a weaker variant of the Caratheodory set for high dimensional data; see Definition~\ref{def:cordCara} and Algorithm~\ref{cordAlg}. This algorithm runs in $O(nd)$ time, which is the optimal time for this task. Using this improved algorithm, a (``coreset'') matrix $S$ as in (ii) above, whose rows are not a scaled subset of rows from $A$, can be computed in a faster (optimal) time.

\item example applications for boosting the performance of \emph{existing} solvers by running them on the matrix $S$ above or its variants for Linear/Ridge/Lasso Regressions and Elastic-net.

\item extensive experimental results on synthetic and real-world data for common LMS solvers of Scikit-learn library with either CPython or Intel's distribution. Either the running time or numerical stability is improved up to two orders of magnitude.

\item open code~\cite{opencode} for our algorithms that we hope will be used for the many other LMS solvers and future research as suggested in our Conclusion section; see Section~\ref{sec:conclude}.
\end{enumerate}

\subsection{Novel approach: Coresets meet Sketches} \label{sec:novelty}
As explained in Section~\ref{sec:relatedWork}, the covariance matrix $A^TA$ of $A$ itself can be considered as a sketch which is relatively less numerically stable to maintain (especially its inverse, as desired by e.g. linear regression). The Caratheodory set, as in Definition~\ref{def:caraSet}, that corresponds to the set of outer products of the rows of $A$ is a coreset whose weighted sum yields the covariance matrix $A^TA$. Moreover, it is more numerically stable but takes much more time to compute; see Theorem~\ref{caraTheorem}.

To this end, we suggest a meta-algorithm that combines these two approaches: sketches and coresets. It may be generalized to other, not-necessarily accurate, $\eps$-coresets and sketches ($\eps>0$); see Section~\ref{sec:conclude}.

\textbf{The input }to our meta-algorithm is 1) a set $P$ of $n$ items, 2) an integer $k \in \br{1,\cdots,n}$ where $n$ is highest numerical accuracy but longest running time, and 3) a pair of coreset and sketch construction schemes for the problem at hand. \\ \textbf{The output } is a coreset for the problem whose construction time is faster than the construction time of the given coreset scheme; see Fig.~\ref{fig:booster}.

\noindent\textbf{Step I: }Compute a balanced partition $\br{P_1,\cdots,P_k}$ of the input set $P$ into $k$ clusters of roughly the same size. While the correctness holds for any such arbitrary partition (e.g. see Algorithm~\ref{theorem:fastCara}), to reduce numerical errors -- the best is a partition that minimizes the sum of loss with respect to the problem at hand.

\noindent\textbf{Step II: }Compute a sketch $S_i$ for each cluster $P_i$, where $i\in\br{1,\cdots,k}$, using the input sketch scheme. This step does not return a subset of $P$ as desired, and is usually numerically less stable.

\noindent\textbf{Step III: }Compute a coreset $B$ for the union $S=S_1\cup\cdots\cup S_k$ of sketches from Step II, using the input coreset scheme. Note that $B$ is not a subset (or coreset) of $P$.

\noindent\textbf{Step IV: }Compute the union $C$ of clusters in $P_1,\cdots,P_k$ that correspond to the selected sketches in Step III, i.e. $C= \bigcup_{S_i\in B } P_i$. By definition, $C$ is a coreset for the problem at hand.

\noindent\textbf{Step V: }Recursively compute a coreset for $C$ until a sufficiently small coreset is obtained. This step is used to reduce running time, without selecting $k$ that is too small.

We then run an existing solver on the coreset $C$ to obtain a faster accurate solution for $P$. Algorithm~\ref{fastCara} and~\ref{theorem:fastCara} are special cases of this meta-algorithm, where the sketch is simply the sum of a set of points/matrices, and  the coreset is the existing (slow) implementation of the Caratheodory set from Theorem~\ref{caraTheorem}.

\noindent\textbf{Paper organization.} In Section~\ref{sec:prelim} we give our notations, definitions and the current state-of-the-art result. Section~\ref{sec:fastCara} presents our main algorithms for efficient computation of the Caratheodory (core-)set and a subset that preserves the inputs covariance matrix, their theorems of correctness and proofs. Later, at section~\ref{sec:fastCaraNew}, we suggest an algorithm that computes a weaker variant of the Caratheodory set in a faster time, which also results in a faster time algorithm for computing a subset that preserves the inputs covariance.
Sections~\ref{sec:caraToLMS},~\ref{sec:cordCaraToLMS}, and~\ref{SVDpca} demonstrate the applications of those algorithms to common LMS solvers and dimensionality reduction algorithms, while Section~\ref{sec:ER} shows the practical usage of this work using extensive experimental results on both real-world and synthetic data via the Scikit-learn library with either CPython or Intel's Python distributions. We conclude the paper with open problems and future work in Section~\ref{sec:conclude}.

\begin{figure}
  \centering
  \includegraphics[width=\textwidth]{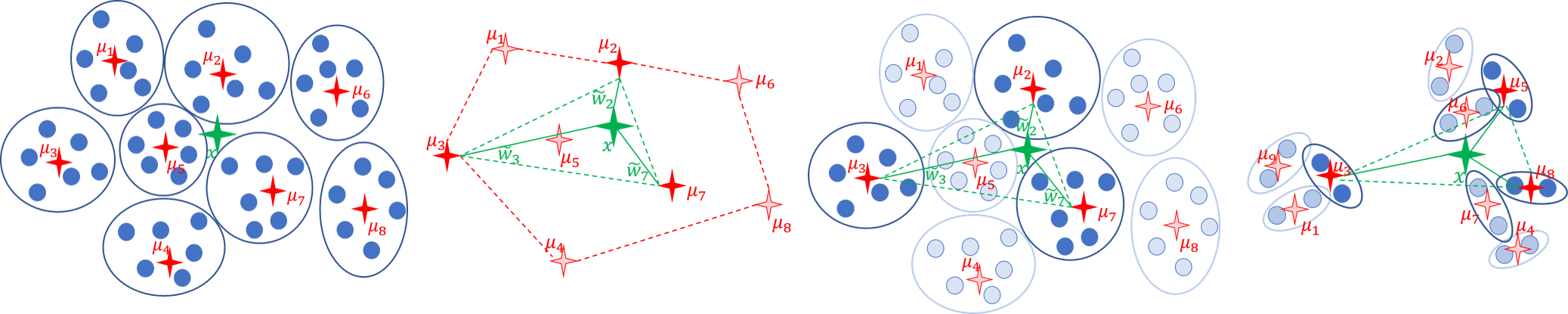}
  \caption{Overview of Algorithm~\ref{fastCara} and the steps in Section~\ref{sec:novelty}. Images left to right: Steps I and II (Partition and sketch steps): A partition of the input weighted set of $n=48$ points (in blue) into $k=8$ equal clusters (in circles) whose corresponding means are $\mu_,\ldots,\mu_8$ (in red). The mean of $P$ (and these means) is $x$ (in green). Step III (Coreset step): Caratheodory (sub)set of $d+1=3$ points (bold red) with corresponding weights (in green) is computed only for these $k=8\ll n$ means. Step IV (Recover step): the Caratheodory set is replaced by its corresponding original points (dark blue). The remaining points in $P$ (bright blue) are deleted. Step V (Recursive step): Previous steps are repeated until only $d+1=3$ points remain. This procedure takes $O(\log n)$ iterations for $k=2d+2$.}
  \label{fig:booster}
\end{figure}

\section{Notation and Preliminaries} \label{sec:prelim}
For a pair of integers $n,d\geq1$, we denote by $\REAL^{n\times d}$ the set of $n\times d$ real matrices, and $[n]=\br{1,\cdots,n}$. To avoid abuse of notation, we use the big $O$ notation where $O(\cdot)$ is a set~\cite{cormen2009introduction}. A \emph{weighted set} is a pair $(P,u)$ where $P=\br{p_1,\cdots,p_n}$ is an ordered finite set in $\REAL^d$, and $u:P \to [0,\infty)$ is a positive \emph{weights function}. We sometimes use a matrix notation whose rows contains the elements of $P$ instead of the ordered set notation.

Given a point $q$ inside the convex hull of a set of points $P$, Caratheodory's Theorem proves that there a subset of at most $d+1$ points in $P$ whose convex hull also contains $q$. This geometric definition can be formulated as follows.
\begin{definition} [Caratheodory set]\label{def:caraSet}
Let $(P,u)$ be a weighted set of $n$ points in $\REAL^d$ such that $\sum_{p\in P}u(p) = 1$.
A weighted set $(S,w)$ is called a \emph{Caratheodory Set} for $(P,u)$ if:
(i) $S \subseteq P$, (ii) its size is $|S|\leq d+1$, (iii) its weighted mean is the same, $\sum_{p\in S} w(p)\cdot p = \sum_{p \in P}u(p)\cdot p$, and (iv) its sum of weights is $\sum_{p\in S} w(p) = 1$.
\end{definition}

Caratheodory's Theorem suggests a constructive proof for computing this set in $O(n^2d^2)$ time~\cite{caratheodory1907variabilitatsbereich, cook1972caratheodory}; see Algorithm~\ref{caraAlg} along with an overview and full proof in Section~\ref{sec:CaraAlg} of the Appendix. 
However, as observed e.g. in~\cite{nasser2015coresets}, it can be computed only for the first $m=d+1$ points, and then be updated point by point in $O(md^2)=O(d^3)$ time per point, to obtain $O(nd^3)$ overall time. This still takes $\Theta(n)$ calls to a linear system solver that returns $x\in\REAL^d$ satisfying $Ax=b$ for a given matrix $A\in \REAL^{(d+1)\times d}$ and vector $b\in \REAL^{d+1}$, in $O(d^3)$ time per call.
\begin{theorem}[\cite{caratheodory1907variabilitatsbereich},~\cite{nasser2015coresets}\label{caraTheorem}]
A Caratheodory set $(S,w)$ can be computed for any weighted set $(P,u)$ where $\sum_{p\in P} u(p) = 1$ in $\tcara(n,d) \in O(1)\cdot \min\br{n^2d^2,nd^3}$ time.
\end{theorem}

\section{Faster Caratheodory Set} \label{sec:fastCara}
In this section, we present our main algorithm that reduces the running time for computing a Caratheodory set from $O(\min\br{n^2d^2,nd^3})$ in Theorem~\ref{caraTheorem} to $O(nd)$ for sufficiently large $n$; see Theorem~\ref{theorem:fastCara}. A visual illustration of the corresponding Algorithm~\ref{fastCara} is shown in Fig.~\ref{fig:booster}. As an application, we present a second algorithm, called \cova, which computes a small weighted subset of a the given input that has the same covariance matrix as the input matrix; see Algorithm~\ref{covAlg}.

\begin{theorem}[Caratheodory-Set Booster] \label{theorem:fastCara}
Let $(P,u)$ be a weighted set of $n$ points in $\REAL^d$ such that $\sum_{p\in P}u(p)=1$, and $k\geq d+2$ be an integer.
Let $(C,w)$ be the output of a call to $\caraf(P,u,k)$; See Algorithm~\ref{fastCara}.
Let $\tcara(k,d)$ be the time it takes to compute a Caratheodory Set for $k$ points in $\REAL^d$, as in Theorem~\ref{caraTheorem}. Then $(C,w)$ is a Caratheodory set of $(P,u)$ that is computed in time $O\left(nd+\tcara(k,d)\cdot \frac{\log n}{\log (k/ d)}\right)$.
\end{theorem}
\begin{proof}
See full proof of Theorem~\ref{theorem:fastCara_proof} in the Appendix.
\end{proof}

\paragraph{Tuning Algorithm~\ref{fastCara} for the fastest running time. } To achieve the fastest running time in Algorithm~\ref{fastCara}, simple calculations show that when $t(k,d) = kd^3$, i.e., when applying the algorithm from~\cite{nasser2015coresets}, $k = ed$ is the optimal value (that achieves the fastest running time), and when $t(k,d) = k^2d^2$, i.e., when applying the original Caratheodory algorithm (Algorithm~\ref{caraAlg} in the Appendix), $k=\sqrt{e}d$ is the value that achieves the fastest running time.

\setcounter{AlgoLine}{0}
\begin{algorithm}
\caption{$\caraf(P,u,k)$; see Theorem~\ref{theorem:fastCara}\label{fastCara}}
\SetKwInOut{Input}{Input}
\SetKwInOut{Output}{Output}
\Input{A set $P$ of $n$ points in $\REAL^d$, a (weight) function $u:P\to [0,\infty)$ such that $\sum_{p\in P}u(p)=1$, and an integer (number of clusters) $k\in \br{1,\cdots,n}$ for the numerical accuracy/speed trade-off.}
\Output{A Caratheodory set of $(P,u)$; see Definition~\ref{def:caraSet}.}

$P := P \setminus \br{p \in P \mid u(p) = 0}$.\label{Line:removeZero}\quad\quad \tcp{Remove all points with zero weight.}
\If{ $|P|\leq d+1$ \label{stoprule}}
{\Return $(P,u)$ \label{returnInput} \quad\quad\tcp{$|P|$ is already small}}

$\br{P_1,\cdots,P_k}:=$ a partition of $P$ into $k$ disjoint subsets (clusters), each contains at most $\ceil{n/k}$ points. \label{compPartition}

\For{every $i\in\br{1,\cdots,k}$}
{
$\displaystyle \mu_i:=\frac{1}{\sum_{q\in P_i}u(q)} \cdot \sum_{p\in P_i}u(p)\cdot p$ \quad\quad \tcp{the weighted mean of $P_i$} \label{compMui}

$u'(\mu_i) := \sum_{p\in P_i}u(p)$ \label{compMuiWeight} \quad\quad \tcp{The weight of the $i$th cluster.}
}

$(\tilde{\mu},\tilde{w}) :=\caras(\br{\mu_1,\cdots,\mu_k},u')$ \label{compSlowCara} \\ \tcp{see Algorithm~\ref{caraAlg} in the Appendix.}

$\displaystyle C:=\bigcup_{\mu_i\in \tilde{\mu}} P_i$ \label{compC} \\\tcp{$C$ is the union over all clusters $P_i \subseteq P$ whose representative $\mu_i$ was chosen for $\tilde{\mu}$.}

\For {every $\mu_i \in \tilde{\mu}$ and $p\in P_i$\label{forRemainingClusters}}
{
$\displaystyle w(p):=\frac{\tilde{w}(\mu_i)u(p)}{\sum_{q\in P_i}u(q)}$ \label{compW} \quad\quad\tcp{assign weight for each point in $C$}
}

$(C,w):=\caraf(C,w,k)$\quad\quad \tcp{recursive call}

\Return $(C,w)$ \label{recursiveCall}
\end{algorithm}

\subsection{Caratheodory Matrix}

\setcounter{AlgoLine}{0}
\begin{algorithm}
\caption{$\cova(A,k)$; see Theorem~\ref{theorem:cov}\label{covAlg}}
\SetKwInOut{Input}{Input}
\SetKwInOut{Output}{Output}
\Input{A matrix $A = (a_1\mid\cdots\mid a_n)^T\in\REAL^{n\times d}$, and an integer $k\in \br{1,\cdots,n}$ for numerical accuracy/speed trade-off.}
\Output{A matrix $S\in\REAL^{(d^2+1)\times d}$ whose rows are scaled rows from $A$, and $\displaystyle A^TA=  S^TS$.}

\For{every $i\in \br{1\cdots,n}$\label{forPcreate}}
{
Set $p_i \in \REAL^{(d^2)}$ as the concatenation of the $d^2$ entries of $a_ia_i^T\in\REAL^{d\times d}$.\\ \tcp{The order of entries may be arbitrary but the same for all points.}
$u(p_i):=1/n$
}
$P:=\big\{p_i \mid i \in \br{1,\cdots,n}\big\}$ \label{forPcreateend} \quad \quad \tcp{$P$ is a set of $n$ vectors in $\REAL^{(d^2)}$.}

$(C,w):=\caraf(P,u,k)$ \label{compCU} \tcp{$C\subseteq P$ and $|C|=d^2+1$ by Theorem~\ref{theorem:fastCara}}
$S:= $ a $(d^2+1)\times d$ matrix whose $i$th row is $\sqrt{n\cdot w(p_i)}\cdot a_i^T$ for every ${p_i\in C}$. \label{compS}

\Return $S$
\end{algorithm}

\begin{theorem} \label{theorem:cov}
Let $A \in \REAL^{n\times d}$ be a matrix, and $k \geq d^2+2$ be an integer. Let $S \in \REAL^{(d^2+1)\times d}$ be the output of a call to $\cova(A,k)$; see Algorithm~\ref{covAlg}. Let $\tcara(k,d)$ be the computation time of \caras{} (Algorithm~\ref{caraAlg}) given $k$ points in $\REAL^{d}$. Then $A^TA = S^TS$. Furthermore, $S$ is computed in $O\left(nd^2+\tcara(k,d^2)\cdot\frac{\log{n}}{\log{(k/d^2))}}\right)$ time.
\end{theorem}
\begin{proof}
See full proof of Theorem~\ref{theorem:cov_proof} in the Appendix.
\end{proof}

\section{Sparsified Caratheodory} \label{sec:fastCaraNew}
The algorithms presented in the previous section managed to compute a lossless compression,  which is a subset of the input data that preserves its covariance. As the experimental results in Section~\ref{sec:ER} show, those algorithms also maintained a very low numerical error, which was either very close or exactly equal to zero. However, their running time  has a polynomial dependency on the dimension $d$, which makes them impractical for some use cases. Therefore, to support high dimensional data, in this section we provide new algorithms which reduce this dependency on $d$ in their running times, by possibly compromising the numerical accuracy.

Streaming data is widely common approach for reducing an algorithm's run time dependency on the number of points $n$, by simply applying the algorithm on chunks of the input, rather than on the entire input at once.
The new algorithms utilize the streaming fashion, but rather on the coordinates (dimension) of the input, rather than chunks of the input. On each such dimensions-subset, the algorithms from the previous section are applied.

The experiments conducted in Section~\ref{sec:ER} demonstrate the expected improvement in running time when using those new and improved algorithms. Fortunately, the numerical error in practice of those new algorithms was not much larger compared to their slower (older) version, which was much lower than the numerical error of the competing methods in most cases.

For an integer $d$ and an integer $k \leq d$, we define $\mathbb{I}_k \subseteq \REAL^{d\times d}$ to be the set of all diagonal matrices $M \in \br{0,1}^{d\times d}$ which contain only ones and zeros and have exactly $k$ ones and $d-k$ zeros along its diagonal.

A Caratheodory set $(C,w)$ of an input weighted set $(P,u)$ requires $C$ to be a subset of $P$; see Definition~\ref{def:caraSet}. In what follows we define a weaker variant called a $k$-Sparse Caratheodory Set. Now, $C$ is not necessarily a subset of the input set $P$. However, we require that every $c\in C$ can obtained by some $p\in P$ after setting $d-k$ of its entries to zero. A $d$-Sparse Caratheodory Set is a Caratheodory set.
\begin{definition} [$k$-Sparse Caratheodory Set] \label{def:cordCara}
Let $(P,u)$ be a weighted set of $n$ points in $\REAL^d$ such that $\sum_{p\in P}u(p) = 1$, and let $k \leq d$ be an integer.
A weighted set $(C,w)$ is called a \emph{$k$-Sparse Caratheodory set} for $(P,u)$ if: (i) for every $c \in C$ there is $p\in P$ and a diagonal matrix $\tilde{I} \in \mathbb{I}_k$ such that $c = \tilde{I}p$ (i.e., $c$ is simply $p$ with some coordinates set to zero), (ii) its size is $|C|\leq \ceil{\frac{d}{k}}\cdot (k+1)$, (iii) its weighted mean is the same, $\sum_{p\in C} w(p)\cdot p = \sum_{p \in P}u(p)\cdot p$, and (iv) its sum of weights is $\sum_{p\in S} w(p) = \ceil{d/k}$.
\end{definition}

\setcounter{AlgoLine}{0}
\begin{algorithm}
\caption{$\caracord(P,u,k_1,k_2)$; see Theorem~\ref{theorem:cordCara}\label{cordAlg}}
\SetKwInOut{Input}{Input}
\SetKwInOut{Output}{Output}
\Input{A set $P = \br{p_1,\cdots,p_n}\subseteq \REAL^{d}$, a weights function $u:P\to [0,\infty)$ such that $\sum_{p\in P} u(p) = 1$, and two integers $k_1,k_2$ for numerical accuracy/speed trade-off such that $k_1 \in \br{\ceil{\frac{d}{k_2}}+2,\cdots,n}$, and $k_2\in \br{1,\cdots,d}$.}
\Output{A $\ceil{d/k_2}$-Sparse Caratheodory set of $(P,u)$; see Definition~\ref{def:cordCara}.}

$\br{I_1,\cdots,I_{k_2}} :=$ a partition of the indices $\br{1,\cdots,d}$ into $k_2$ disjoint subsets, each containing at most $\ceil{d/k_2}$ indices. \label{Line:partitionCord}

For every $p\in P$ and $j\in [k_2]$ define $p^j \in \REAL^{|I_j|}$ as the point containing only the coordinates of $p\in P$ whose indices are in $I_j$. \\\tcp{$p^j$ contains a subset of the coordinates of $p$, whose indices are in $I_j$.}

$C := \emptyset$

\For{every $j\in \br{1,\cdots,k_2}$\label{line:dimIter}}
{

$P^j := \br{p^j \mid p \in P}$ \label{Line:compPj} \tcp{$P^j$ contains all the points of $P$, when taking only a subset of their coordinates.}





$u^j(p^j) = u(p)$ for every $p\in P$. \label{Line:compuj}

$(C^j,w^j) := \caraf(P^j,u^j,k_1)$. \label{Line:callBoost} \tcp{$C^j \subseteq P^j$ and $|C^j| \leq \ceil{\frac{d}{k_2}}+1$ by Theorem~\ref{theorem:fastCara}.}

For every $c \in C^j$ define $\hat{c} \in \REAL^d$ to be a vector of zeros in the coordinates $\br{1,\cdots,d}\setminus I_j$, and plug the coordinates of $c$ into indices $I_j$ of $\hat{c}$, and let $\hat{C}^j = \br{\hat{c} \mid c \in C^j}$. \label{Line:padVecs} \tcp{transform $c$ back into $\REAL^d$ by adding zeros in specific locations.}

$w(\hat{c}) := w^j(c)$ for every $c \in C^j$. \label{Line:compw} \tcp{set the weight of the padded vector to be the weight of the original vector.}

$C = C\cup \hat{C}^j$

}

\Return $(C,w)$
\end{algorithm}

\begin{theorem}\label{theorem:cordCara}
Let $(P,u)$ be a weighted set of $n$ points in $\REAL^d$ such that $\sum_{p\in P}u(p)=1$, and $k_1,k_2,d'$ be three integers such that $k_2 \in \br{1,\cdots,d}$, $d' = \ceil{\frac{d}{k_2}}$, and $k_1 \in \br{d'+2,\cdots,n}$.
Let $(C,w)$ be the output of a call to $\caracord(P,u,k_1,k_2)$; See Algorithm~\ref{cordAlg}.
Let $\tcara(k_1,d')$ be the time it takes to compute a Caratheodory Set for $k_1$ points in $\REAL^{d'}$, as in Theorem~\ref{caraTheorem}. Then $(C,w)$ is a $d'$-Sparse Caratheodory set of $(P,u)$ that is computed in time $O\left(nd+\tcara(k_1,d')\cdot \frac{k_2\log n}{\log (k_1/d')}\right)$.
\end{theorem}
\begin{proof}
See full proof of Theorem~\ref{theorem:cordCara2} in the Appendix.
\end{proof}

\paragraph{Tuning Algorithm~\ref{fastCara} for the fastest running time. } To achieve the fastest running time in Algorithm~\ref{cordAlg}, simple calculations show that plugging, e.g., $t(k,d) =kd^3 $, i.e., when applying the algorithm from~\cite{nasser2015coresets}, $k_2 = d$ and $k_1=4$ yields the optimal running time of $O(nd)$.

\subsection{Sparsified Caratheodory Matrix}

Recall that the covariance $A^TA \in \REAL^{d\times d}$ of a matrix $A = (a_1\mid\cdots\mid a_n)^T\in\REAL^{n\times d}$ is equal to the sum $\sum_{i=1}^n a_ia_i^T$. Using the SVD $A^TA = UDV^T$ of the covariance matrix, one can compute a matrix $S = \sqrt{D} V^T \in \REAL^{d\times d}$ of only $d$ rows whose covariance is the same as $A$, i.e., $S^TS = A^TA$. Observe that this process requires computing the sum of $n$ matrices of size $d\times d$.

In this section, we provide an algorithm which computes such a matrix $S$ by summing over only $O(d^2)$ sparse $d\times d$ matrices. This algorithm requires the same computational time as the previous algorithm, but is more numerically stable due to summing over only a small number of sparse matrices; see Section~\ref{sec:ER} for such comparisons.

\setcounter{AlgoLine}{0}
\begin{algorithm}
\caption{$\cordCova(A,k_1,k_2)$; see Theorem~\ref{theorem:cordCov}\label{cordCovAlg}}
\SetKwInOut{Input}{Input}
\SetKwInOut{Output}{Output}
\Input{A matrix $A = (a_1\mid\cdots\mid a_n)^T\in\REAL^{n\times d}$, and two integers $k_1,k_2$ for numerical accuracy/speed trade-off such that $k_2 \in \br{1,\cdots,d^2}$ and $k_1\in \br{\ceil{\frac{d^2}{k_2}}+2,\cdots,n}$.}
\Output{A matrix $S\in\REAL^{d\times d}$ such that $A^TA=  S^TS$.}

\For{every $i\in \br{1\cdots,n}$\label{forPcreate}}
{
Set $p_i \in \REAL^{d^2}$ as the column stacking of the $d^2$ entries of $a_ia_i^T\in\REAL^{d\times d}$.\\ \tcp{The order of entries may be arbitrary but the same for all points.}
$u(p_i):=1/n$
}
$P:=\big\{p_i \mid i \in \br{1,\cdots,n}\big\}$ \label{forPcreateend} \quad \quad \tcp{$P$ is a set of $n$ vectors in $\REAL^{(d^2)}$.}

$(C,w):=\caracord(P,u,k_1,k_2)$ \label{compCU} \tcp{see Algorithm~\ref{cordAlg}. $C\subseteq P$ and $|C| \in O(d^2+k_2)$ by Theorem~\ref{theorem:cordCara} and Definition~\ref{def:cordCara}.}

$c' := n\cdot \sum_{c\in C} w(c)c \in \REAL^{d^2}$ \label{Line:compctag} \tcp{The weighted sum of $(C,w)$.}

Set $C' \in \REAL^{d\times d}$ as the matrix obtained by reshaping $c'$ into a matrix \label{Line:reshapectag}\\ \tcp{Inverse column-stacking operation.}

Set $S := \sqrt{D}V^T \in \REAL^{d\times d}$ where $C' = UDV^T$ is the thin Singular Value Decomposition of $C'$. \label{Line:compS}

\Return $S$
\end{algorithm}

\begin{theorem} \label{theorem:cordCov}
Let $A \in \REAL^{n\times d}$ be a matrix, and $k_1,k_2,d'$ be three integers such that $k_2 \in \br{1,\cdots,d^2}$, $d' = \ceil{\frac{d^2}{k_2}}$, and $k_1 \in \br{d'+2,\cdots,n}$. Let $S \in \REAL^{d\times d}$ be the output of a call to $\cordCova(A,k_1,k_2)$; see Algorithm~\ref{cordCovAlg}. Let $\tcara(k_1,d')$ be the time it takes to compute a Caratheodory Set for $k_1$ points in $\REAL^{d'}$, as in Theorem~\ref{caraTheorem}. Then $A^TA = S^TS$. Furthermore, $S$ is computed in $O\left(nd^2+\tcara(k_1,d')\cdot \frac{k_2\log n}{\log (k_1/d')}\right)$ time.
\end{theorem}
\begin{proof}
See full proof of Theorem~\ref{theorem:cordCov2} in the Appendix.
\end{proof}

\section{From Caratheodory to LMS Solvers}\label{sec:caraToLMS}

In this section, we first show how Algorithm~\ref{covAlg} can be used to boost the running time of LMS solvers (Lasso/Ridge/Linear/Elastic-net regression) without compromising the accuracy at all. Then, in Section~\ref{sec:cordCaraToLMS}, we show how to leverage Algorithm~\ref{cordCovAlg}, instead of Algorithm~\ref{covAlg}, to boost the running time of LMS solvers potentially even more, in the cost of a potential decrease in numerical accuracy. As the experimental results in Section~\ref{sec:ER} show, although in some cases Algorithm~\ref{cordCovAlg} introduces an additional small numerical error, it still outperforms the competing compression algorithms common used in practice, both as of running time and accuracy.

Before, we remind the reader that LMS solvers use cross validation techniques to select the best hyper parameter values, such as $\alpha$ and $\rho$ in table~\ref{appstable}. In what follows we first explain about the $m$-folds cross validation, then we show how to construct a coreset for different LMS solvers while supporting the the $m$-folds cross validation.

\paragraph{$m$-folds cross validation (CV). } We briefly discuss the CV technique which is utilized in common LMS solvers. Given a parameter $m$ and a set of real numbers $\alphas$, to select the optimal value $\alpha\in \alphas$ of the regularization term, the existing Python's LMS solvers partition the rows of $A$ into $m$ folds (subsets) and run the solver $m\cdot |\alphas|$ times, each run is done on a concatenation of $m-1$ folds (subsets) and $\alpha\in \alphas$, and its result is tested on the remaining ``test fold''. Finally, the cross validation returns the parameter ($\alpha\in \alphas$) that yield the optimal (minimal) mean value on the test folds; see~\cite{kohavi1995study} for details.

\paragraph{From Caratheodory Matrix to LMS solvers.}

As stated in Theorem~\ref{theorem:cov}, Algorithm~\ref{covAlg} gets an input matrix $A\in\REAL^{n\times d}$ and an integer $k>d+1$, and returns a matrix $S\in \REAL^{(d^2+1)\times d}$ of the same covariance $A^TA=S^TS$, where $k$ is a parameter for setting the desired numerical accuracy. To "learn" a given label vector $b\in\REAL^n$, Algorithm~\ref{LMSCoreset} partitions the matrix $A'=(A\mid b)$ into $m$ partitions, computes a subset for each partition that preserves its covariance matrix, and returns the union of subsets as a pair $(C,y)$ where $C\in \REAL^{(m(d+1)^2+m)\times d}$ and $y\in \REAL^{m(d+1)^2+m}$. For $m=1$ and every $x\in\REAL^d$,
\begin{equation}\label{eq1fold}
\begin{split}
\norm{Ax-b}
=\norm{A'(x \mid -1)^T}
=\norm{(C \mid y)(x \mid -1)^T}
=\norm{Cx-y},
\end{split}
\end{equation}
where the second and third equalities follow from Theorem~\ref{theorem:cov} and the construction of $C$, respectively. This enables us to replace the original pair $(A,b)$ by the smaller pair $(C,y)$ for the solvers in Table~\ref{appstable} as in Algorithms~\ref{LstsqOurAlg}--\ref{ElasticOurAlg}. A scaling factor $\beta$ is also needed in Algorithms~\ref{LassoOurAlg}--\ref{ElasticOurAlg}.

To support CV with $m>1$ folds, Algorithm~\ref{LMSCoreset} computes a coreset for each of the $m$ folds (subsets of the data) in Line~\ref{compCov} and concatenates the output coresets in Line~\ref{concatCoresets}. Thus,~\eqref{eq1fold} holds similarly for each fold (subset) when $m>1$.

\setcounter{AlgoLine}{0}
\begin{algorithm}
\caption{$\calcCoreset(A,b,\folds,k)$}\label{LMSCoreset}
{\begin{minipage}{\textwidth}
\begin{tabbing}
\textbf{Input:} \quad\=A matrix $A\in\REAL^{n\times d}$, a vector $b\in \REAL^n$, a number (integer) $m$ of cross-validation folds,\\\>and an integer $k\in \br{1,\cdots,n}$ that denotes accuracy/speed trade-off.\\
\textbf{Output:} \>A matrix $C\in\REAL^{O(md^2)\times d}$ whose rows are scaled rows from $A$, and a vector $y\in\REAL^d$.
\end{tabbing}\end{minipage}}

\nl $A':= (A\mid b )$ \quad\quad\tcp{A matrix $A' \in \REAL^{n\times (d+1)}$}
$\br{A'_1, \cdots ,A'_{\folds}}:=$ a partition of the rows of $A'$ into $\folds$ matrices, each of size $(\frac{n}{\folds}) \times (d+1)$

\For{every $i\in\br{1,\cdots,\folds}$} {

	$S_i:= \cova(A'_i,k)$ \label{compCov}\quad \tcp{see Algorithm~\ref{covAlg}} 

}
$S:=(S_1^T | \cdots | S_{\folds}^T)^T$ \label{concatCoresets} \tcp{concatenation of the $\folds$ matrices into a single matrix of $m(d+1)^2+m$ rows and $d+1$ columns}

$C:=$ \= the first $d$ columns of $S$

$y:=$ \= the last column of $S$

\Return $(C,y)$
\end{algorithm}
\begin{minipage}[t]{6.7cm}
  \setcounter{AlgoLine}{0}
  \begin{algorithm}[H]
    \caption{$\LstsqOurAlg(A,b,m,k)$}\label{LstsqOurAlg}
\nl $(C,y):=\calcCoreset(A,b,m,k)$

\nl $x^* := \lstsqq(C,y)$

\nl \Return $x^*$
  \end{algorithm}
\end{minipage}%
\hfill
\begin{minipage}[t]{7cm}
  \setcounter{AlgoLine}{0}
  \begin{algorithm}[H]
    \caption{$\RidgeOurAlg(A,b,\alphas,\folds,k)$}\label{RidgeOurAlg}
\nl $(C,y):=\calcCoreset(A,b,\folds,k)$

\nl  $(x,\alpha) := \ridge(C,y,\alphas,m)$

\nl \Return $(x,\alpha)$

  \end{algorithm}
\end{minipage}
\begin{minipage}[t]{6.7cm}
  \setcounter{AlgoLine}{0}
    \begin{algorithm}[H]
    \caption{$\LassoOurAlg(A,b,\alphas,\folds,k)$}\label{LassoOurAlg}
\nl $(C,y):=\calcCoreset(A,b,\folds,k)$

\nl $\beta := \sqrt{\big({\folds \cdot \big(d+1)^2+\folds\big)}/n}$

\nl  $(x,\alpha) := \lasso(\beta \cdot C, \beta \cdot y,\alphas,m)$

\nl \Return $(x,\alpha)$
  \end{algorithm}
\end{minipage}%
\hfill
\begin{minipage}[t]{7cm}
  \setcounter{AlgoLine}{0}
  \begin{algorithm}[H]
    \caption{$\ElasticOurAlg(A,b,\folds,\alphas,\rho,k)$}\label{ElasticOurAlg}
\nl $(C,y):=\calcCoreset(A,b,\folds,k)$

\nl $\beta := \sqrt{\big({\folds \cdot \big(d+1)^2+\folds\big)}/n}$

\nl $(x,\alpha):=\elastic(\beta \cdot C, \beta \cdot y,\alphas,\rho,m)$

\nl \Return $(x,\alpha)$
  \end{algorithm}
\end{minipage}

\section{From Sparse Caratheodory to LMS Solvers} \label{sec:cordCaraToLMS}
In this section, we replace Algorithm~\ref{LMSCoreset} from the previous section by Algorithms~\ref{LMSCoresetnew}, which utilizes Algorithm~\ref{cordCovAlg} instead of Algorithm~\ref{covAlg} to reduce the running time's polynomial dependency on $d$. The fastest running time for Algorithms~\ref{LMSCoresetnew}, after tuning its parameters, is $O(nd^2)$.  

Algorithm~\ref{LMSCoresetnew} also partitions the input matrix $A' = (A \mid b)$ from the previous section into $m$ folds. It then computes, for each fold, a set of only $d$ rows that maintains the covariance of this fold using Algorithm~\ref{cordCovAlg} (instead of the $(d+1)^2$ subset of rows from the previous section). The output is the union $(C,y)$ of all those subsets where $C \in \REAL^{md \times d}$ and $y \in \REAL^{md}$. Therefore, $C$ and $y$ here (i) satisfy~\eqref{eq1fold} for any $m\geq 1$, (ii) are smaller than those computed in the previous section, but (iii) they are not a subset of $A$ and $b$ respectively.

\newcommand{\calcCoresetnew}{\textsc{LMS-Coreset++}}
\setcounter{AlgoLine}{0}
\begin{algorithm}
\caption{$\calcCoresetnew(A,b,\folds,k_1,k_2)$}\label{LMSCoresetnew}
{\begin{minipage}{\textwidth}
\begin{tabbing}
\textbf{Input:} \quad\=A matrix $A\in\REAL^{n\times d}$, a vector $b\in \REAL^n$, a number (integer) $m$ of cross-validation folds,\\\>and two integers $k_1,k_2$ for numerical accuracy/speed trade-off such that \\\>$k_2 \in \br{1,\cdots,(d+1)^2}$ and $k_1\in \br{\ceil{\frac{(d+1)^2}{k_2}}+2,\cdots,n}$.\\
\textbf{Output:} \>A matrix $C\in\REAL^{O(md)\times d}$, and a vector $y\in\REAL^d$.
\end{tabbing}\end{minipage}}

\nl $A':= (A\mid b )$ \quad\quad\tcp{A matrix $A' \in \REAL^{n\times (d+1)}$}
$\br{A'_1, \cdots ,A'_{\folds}}:=$ a partition of the rows of $A'$ into $\folds$ matrices, each of size $(\frac{n}{\folds}) \times (d+1)$

\For{every $i\in\br{1,\cdots,\folds}$} {

	$S_i:= \cordCova(A'_i,k_1,k_2)$ \label{compCov}\quad \tcp{see Algorithm~\ref{cordCovAlg}} 

}
$S:=(S_1^T | \cdots | S_{\folds}^T)^T$ \label{concatCoresets} \tcp{concatenation of the $\folds$ matrices into a single matrix of $md$ rows and $d+1$ columns}

$C:=$ \= the first $d$ columns of $S$

$y:=$ \= the last column of $S$

\Return $(C,y)$
\end{algorithm}
\newcommand{\LstsqOurAlgnew}{\textsc{LinReg-Boost++}}
\newcommand{\LassoOurAlgnew}{\textsc{Lassocv-Boost++}}
\newcommand{\RidgeOurAlgnew}{\textsc{Ridgevc-Boost++}}
\newcommand{\ElasticOurAlgnew}{\textsc{Elasticv-Boost++}}
\begin{minipage}[t]{6.7cm}
  \setcounter{AlgoLine}{0}
  \begin{algorithm}[H]
    \caption{$\LstsqOurAlgnew(A,b,m,k_1,k_2)$}\label{LstsqOurAlgnew}
\nl $(C,y):=\calcCoresetnew(A,b,m,k_1,k_2)$

\nl $x^* := \lstsqq(C,y)$

\nl \Return $x^*$
  \end{algorithm}
\end{minipage}%
\hfill
\begin{minipage}[t]{7cm}
  \setcounter{AlgoLine}{0}
  \begin{algorithm}[H]
    \caption{$\RidgeOurAlgnew(A,b,\alphas,\folds,k_1,k_2)$}\label{RidgeOurAlgnew}
\nl $(C,y):=\calcCoresetnew(A,b,\folds,k_1,k_2)$

\nl  $(x,\alpha) := \ridge(C,y,\alphas,m,k_1,k_2)$

\nl \Return $(x,\alpha)$

  \end{algorithm}
\end{minipage}
\begin{minipage}[t]{6.7cm}
  \setcounter{AlgoLine}{0}
    \begin{algorithm}[H]
    \caption{$\LassoOurAlgnew(A,b,\alphas,\folds,k_1,k_2)$}\label{LassoOurAlgnew}
\nl $(C,y):=\calcCoresetnew(A,b,\folds,k_1,k_2)$

\nl $\beta := \sqrt{\frac{md}{n}}$

\nl  $(x,\alpha) := \lasso(\beta \cdot C, \beta \cdot y,\alphas,m)$

\nl \Return $(x,\alpha)$
  \end{algorithm}
\end{minipage}%
\hfill
\begin{minipage}[t]{7cm}
  \setcounter{AlgoLine}{0}
  \begin{algorithm}[H]
    \caption{$\ElasticOurAlgnew(A,b,\folds,\alphas,\rho,k_1,k_2)$}\label{ElasticOurAlgnew}
\nl $(C,y):=\calcCoresetnew(A,b,\folds,k_1,k_2)$

\nl $\beta := \sqrt{\frac{md}{n}}$

\nl $(x,\alpha):=\elastic(\beta \cdot C, \beta \cdot y,\alphas,\rho,m)$

\nl \Return $(x,\alpha)$
  \end{algorithm}
\end{minipage}

\section{Coresets for SVD and PCA}\label{SVDpca}

In this section, we show how to leverage Algorithm~\ref{covAlg} in order to construct coresets for dimensionality reduction algorithms such as the widely used Principal Component Analysis (PCA) and Singular Value Decomposition (SVD). We first briefly define the $j$-SVD and $j$-PCA problems. 
We then demonstrate how a coreset for the $j$-SVD problem can be obtained using Algorithm~\ref{covAlg}; see Observation~\ref{obesrv:svd}. Finally, we suggest a coreset construction algorithm for the $j$-PCA problem; see Algorithm~\ref{PCA-CORESET} and Observation~\ref{obesrv:pca}. 

LMS solvers usually support data which is not centralized around the origin. The PCA is closely related to this uncetralized-data case, since it aims to find an affine subspace (does not intersect the origin), which best fits the data. Therefore, a coreset for PCA, as presented in this section, can also serve as a coreset for LMS solvers with uncentralized data. In common coding libraries, such as SKlearn, this property is usually referred to by a flag called \emph{fit\_intercept}.

\paragraph{$j$-SVD.}
In the $j$-SVD problem, we are given an input matrix $A\in \REAL^{n\times d}$ and an iteger $j\geq 1$, and the goal is to compute the linear (non-affine) $j$-dimensional subspace that minimizes its sum of squared distances to the rows of $A$. 
Here, a matrix $C\in \REAL^{m \times d}$ is a coreset for the input matrix $A$ if it satisfies the following pair of properties: (i) The rows of $C$ are scaled rows of $A$, and (ii) the sum of the squared distances from every (non-affine) $j$-dimensional subspace to either the rows of $C$ or the rows of $A$ is approximately the same, up to some multiplicative factor.
For the coreset to be effective, we aim to compute such $C$ where $m\ll n$.

Formally, let $H$ be a (non-affine) $j$-dimensional subspace of $\REAL^d$. As explained at~\cite{maalouf2019tight}, every such subspace $H$ is spanned by the column space of a matrix $X\in \REAL^{d\times j}$ whose columns are orthonormal, i.e., $X^TX=I_j$. Given this matrix $X$, for every $i\in\br{1,\cdots, n}$ the squared distance from the $i$th row $a_i$ of $A$ to $H$ can be written as
\[
\norm{a_i^T-a_i^TXX^T}_2^2.
\]

Let $Y\in \REAL^{d\times(d-j)}$ to be a matrix whose columns are mutually orthogonal unit vectors that span the orthogonal complement subspace of $H$ (i.e., $Y^TY=I_{(d-j)}$ and $[X \mid Y]^T[X \mid Y]=I_d$). The squared distance from the $i$th row $a_i$ of $A$ to $H$ can now be written as $\norm{a_iY}^2_2$; See full details in Section~{3} at~\cite{maalouf2019tight}. Hence, the sum of squared distance from the rows of $A$ to the $j$-subspace $H$ is equal to
\begin{align}
 \smi \norm{a_iY}^2_2 = \norm{AY}_F^2. \label{useitinsens}
\end{align}


\paragraph{$j$-PCA.}
More generally, in the $j$-PCA problem, the goal is to compute the \emph{affine} $j$-dimensional subspace that minimizes its sum of squared distances to the rows of $A$, over every $j$-dimensional subspace that may be translated from the origin of $\REAL^d$. Formally, an affine $j$-dimensional subspace $H$ is represented by a pair $(X,\ell)$ where $X\in \REAL^{d\times j}$ is an orthogonal matrix, and $\ell$ is a vector in $\REAL^d$ that represents the translation of the subspace from the origin. Hence, the sum of squared distance from the rows of $A$ to the affine $j$-dimensional subspace $H$ is 
\begin{align}
\smi \norm{(a_i-\ell) -(a_i-\ell)XX^T}^2.    
\end{align}

As above, by letting $Y\in \REAL^{d\times(d-j)}$ be an orthogonal matrix whose rows span the orthogonal complement subspace of $H$, the sum of squared distances from the rows of $A$ to $H$ is now equal to
\begin{align*}
 \smi \norm{(a_i-\ell^T)Y}^2_2.
\end{align*}

\newcommand{\pcacoreset}{\textsc{PCA-CORESET}}
\newcommand{\smil}{\sum_{i=1}^l}
\setcounter{AlgoLine}{0}
\begin{algorithm}[H]
\caption{$\pcacoreset(A,{k})$}\label{PCA-CORESET}
\SetKwInOut{Input}{Input}
\SetKwInOut{Output}{Output}
\Input{A matrix $A\in\REAL^{n\times d}$, and an integer ${k}\in \br{1,\cdots,n}$ that denotes accuracy/speed trade-off.}
\Output{A matrix $C\in\REAL^{l \times d}$ whose rows are scaled rows in $A$, and a weights function $w$, where $l = (d+1)^2 +1$. See Observation~\ref{obesrv:pca}.}
$l = (d+1)^2 +1$

$A' := [ A\mid (1,\cdots ,1)^T]$\label{A'line}

$S' := \cova(A',{k})$

Identify the $i$th row of $S'$ by $s'_i = (s_i^T \mid z_i)$, where $s_i \in \REAL^{d}$ and $z_i \in \REAL$

Set $C \in \REAL^{l\times d}$ to be a matrix whose $i$th row is $c_i := s_i^T/z_i$.\label{Sdefff}

$w(c_i) := z_i^2$ for every $i \in l$.\label{wdefff} 

\Return $(C,w)$
\end{algorithm}

\begin{observation}[$j$-SVD coreset] \label{obesrv:svd}
Let $A \in \REAL^{n\times d}$ be a matrix, $j\in\br{1,\cdots,d-1}$ be an integer, and $k \geq d^2+2$. Let $S \in \REAL^{(d^2+1)\times d}$ be the output of a call to $\cova(A,{k})$; see Algorithm~\ref{covAlg}.  Then for every matrix $Y\in \REAL^{d\times(d-j)}$ such that $Y^TY=I_{(d-j)}$, we have that
$\norm{AY}^2_F = \norm{SY}^2_F$.
\end{observation}
\begin{proof}
See full proof of Observation~\ref{obesrv:svd2} in the Appendix.
\end{proof}

\begin{observation}[$j$-PCA coreset] \label{obesrv:pca}
Let $A = (a_1 \mid \cdots \mid a_n)^T \in \REAL^{n\times d}$ be a matrix, and $j\in\br{1,\cdots,d-1}$, $l = (d+1)^2 + 1$, and $k \geq d^2+2$ be integers. Let $(C,w)$ be the output of a call to $\pcacoreset(A,{k})$; see Algorithm~\ref{PCA-CORESET}, where $C = (c_1 \mid \cdots \mid c_l)^T \in \REAL^{l \times d}$ and $w\in \REAL^{l}$.
 Then for every matrix $Y\in \REAL^{d\times(d-j)}$ such that $Y^TY=I$, and a vector $\ell \in \REAL^{d}$ we have that
\begin{align*}
 \smi \norm{(a_i-\ell^T)Y}^2_2= \sum_{i=1}^l w_i \norm{(c_i-\ell^T)Y}^2_2,
\end{align*}
\end{observation}
\begin{proof}
See full proof of Observation~\ref{obesrv:pca2} in the Appendix.
\end{proof}

\section{Experimental Results} \label{sec:applications} \label{sec:ER}

\begin{table}
\centering
\begin{adjustbox}{width=\textwidth}
\begin{tabular}{|c|c|c|c|}
\hline & & &\\[-8pt]
\bf Solver  &  \bf Objective function &  \bf Python's Package & \bf Example Python's solver \\[2pt]
\hline\\[-8pt]
Linear regression~\cite{bjorck1967solving}  & $\displaystyle {\norm{Ax-b}_2^2}$ & \scipy & \lstsqq$(A,b)$ \\[10pt]
\hline & & &\\[-8pt]
Ridge regression~\cite{hoerl1970ridge}   & $\displaystyle{{\norm{Ax-b}_2^2}+\alpha\norm{x}_2^2}$& \sklearn &  \ridge($A,b,\alphas,m$)  \\[10pt]
\hline& & &\\[-8pt]
Lasso regression~\cite{tibshirani1996regression}  & $\displaystyle{\frac{1}{2n}{\norm{Ax-b}_2^2}+\alpha\norm{x}_1}$ &\sklearn  & \lasso($A,b,\alphas,m$)  \\[10pt]
\hline & & &\\[-8pt]
Elastic-Net regression~\cite{zou2005regularization}  &$\displaystyle\frac{1}{2n}{\norm{Ax-b}_2^2}+ \rho\alpha\norm{x}_2^2 +\frac{(1-\rho)}{2}\alpha\norm{x}_1$   & \sklearn &  \elastic($A,b,\alphas,\rho,m$)\\[10pt]
\hline
\end{tabular}
\end{adjustbox}
\caption{Four LMS solvers that were tested with Algorithm~\ref{LMSCoreset}. Each procedure gets a matrix $A\in\REAL^{n\times d}$, a vector $b\in\REAL^n$ and aims to compute $x\in\REAL^d$ that minimizes its objective function. Additional regularization parameters include $\alpha>0$ and $\rho\in [0,1]$. The Python's solvers use $m$-fold cross validation over every $\alpha$ in a given set $\alphas\subseteq[0,\infty)$.}\label{appstable}
\end{table}

In this section we apply our fast construction of the (Sparse) Carathoodory Set $S$ from the previous sections to boost the running time of common LMS solvers in Table~\ref{appstable} by a factor of tens to hundreds, or to improve their numerical accuracy by a similar factor to support, e.g., 32 bit floating point representation as in Fig.~\ref{fig:accuracy}. This is by running the given solver as a black box on the small matrix $C$ that is returned by Algorithms~\ref{LstsqOurAlg}--\ref{ElasticOurAlg} and Algorithms~\ref{LstsqOurAlgnew}--\ref{ElasticOurAlgnew}, which is based on $S$. That is, our algorithm does not compete with existing solvers but relies on them, which is why we called it a "booster". Open code for our algorithms is provided~\cite{opencode}.



\noindent\textbf{The experiments. }We applied our \calcCoreset{} and \calcCoresetnew{} coresets from Algorithms~\ref{LMSCoreset} and~\ref{LMSCoresetnew} on common Python's SKlearn LMS-solvers that are described in Table~\ref{appstable}.
Most of these experiments were repeated twice: using the default CPython distribution~\cite{wiki:cpython} and Intel's distribution~\cite{inteldist} of Python. All the experiments were conducted on a standard Lenovo Z70 laptop with an Intel i7-5500U CPU @ 2.40GHZ and 16GB RAM.
We used the $3$ following real-world datasets from the UCI Machine Learning Repository~\cite{Dua:2019}:
\begin{enumerate}[(i)]
\item 3D Road Network (North Jutland, Denmark)~\cite{kaul2013building}. It contains $n=434874$ records. We used the $d=2$ attributes: ``Longitude'' [Double] and ``Latitude'' [Double] to predict the attribute ``Height in meters'' [Double].
\item Individual household electric power consumption~\cite{dataset:power}. It contains $n=2075259$ records. We used the $d=2$ attributes: ``global active power'' [kilowatt - Double], ``global reactive power'' [kilowatt - Double]) to predict the attribute ``voltage'' [volt - Double].
\item House Sales in King County, USA~\cite{dataset:sales}. It contains $n=21,600$ records. We used the following $d=8$ attributes: ``bedrooms'' [integer], ``sqft living'' [integer], ``sqft lot'' [integer], ``floors'' [integer], ``waterfront'' [boolean], ``sqft above'' [integer], ``sqft basement'' [integer], ``year built'' [integer]) to predict the ``house price'' [integer] attribute.
\item Year Prediction Million Song Dataset~\cite{Bertin-Mahieux2011}. It contains $n = 515345$ records in $d=90$ dimensional space. We used the attributes $2$ till $90$ [Double] to predict the song release year [Integer] (first attribute).
\end{enumerate}

The synthetic data consists of an $n\times d$ matrix $A$ and vector $b$ of length $n$, both of uniform random entries in $[0,1000]$. As expected by the analysis, since our compression introduces no error to the computation accuracy, the actual values of the data had no affect on the results, unlike the size of the input which affects the computation time. Table~\ref{expstable} summarizes the experimental results.


\subsection{Competing methods}
We now present other sketches for improving the practical running time of LMS solvers; see discussion in Section~\ref{sec:discuss}.\\
\textbf{SKETCH + CHOLESKY} is a method which simply sums the $1$-rank matrices of outer products of rows in the input matrix $A'=(A\mid b)$ which yields its covariance matrix $B=A'^TA'$. The Cholesky decomposition $B=L^TL$ then returns a small matrix $L\in\REAL^{d\times d}$ that can be plugged to the solvers, similarly to our coreset.\\
\textbf{SKETCH + SVD} is a method which simply sums the $1$-rank matrices of outer products of rows in the input matrix $A'=(A\mid b)$, which yields its covariance matrix $B=A'^TA'$. The SVD decomposition $B=UDV^T$ is then applied to return a small matrix $\sqrt{D}V^T\in\REAL^{d\times d}$ that can be plugged to the solvers, similarly to our coreset.\\
\textbf{SKETCH + INVERSE} is applied in the special case of linear regression, where one can avoid applying the Cholesky decomposition and can compute the solution $(A^TA)^{-1}A^Tb$ directly after maintaining $A^TA$ and $A^Tb$ for the data seen so far.

\subsection{Discussion} \label{sec:discuss}
\paragraph{Practical parameter tuning. }As analyzed in Section~\ref{sec:fastCaraNew}, the theoretically optimal value for $k_2$ (for Algorithm~\ref{cordAlg}) would be $k_2 = d$. When considering Algorithms~\ref{RidgeOurAlgnew}--\ref{ElasticOurAlgnew}, where the dimension of the data to be compressed is $(d+1)^2$, it is straightforward that the optimal theoretical value is $k_2 = (d+1)^2$. However, in practice, this might not be the case due to the following tradeoff: a larger value of $k_2$ in practice means a larger number of calls to the subprocedure $\caraf$, though the dimension of the data in each call is smaller (i.e., smaller theoretical computational time), and vice versa. In our experiments we found that setting $k_2$ to be its maximum possible value ($(d+1)^2$) divided by some constant ($12$ in our case) yields the fastest running time; see Table~\ref{expstable}.

\paragraph{Running time. } Consider Algorithm~\ref{LMSCoreset}. The number of rows in the reduced matrix $C$ is $O(d^2)$, which is usually much smaller than the number $n$ of rows in the original matrix $A$. This also explains why some coresets (dashed red line) failed for small values of $n$ in Fig.~\ref{fig:lassoD},\ref{fig:elasticD},\ref{fig:lasso_intel} and~\ref{fig:elastic_intel}. The construction of $C$ takes $O(nd^2 +poly(d))$. 
Now consider the improved Algorithm~\ref{LMSCoresetnew}. The number of rows in the reduced matrix $C$ is only $O(d)$ and requires only $O(nd^2)$ time to compute for some tuning of the parameters as discussed in Section~\ref{sec:fastCaraNew}.
Solving linear regression takes the same time, with or without the coreset. However, the constants hidden in the $O$ notation are much smaller since the time for computing $C$ becomes neglectable for large values of $n$, as shown in Fig.~\ref{fig:lstsq}. We emphasize that, unlike common coresets, there is \emph{no accuracy loss} due to the use of our coreset, ignoring $\pm 10^{-15}$ additive errors/improvements. The improvement in running time due to our booster is in order of up to x10 compared to the algorithm's running time on the original data, for both small and large values of the dimension $d$, as shown in Fig.~\ref{fig:realData1}--\ref{fig:realData2_intel}, and~\ref{fig:realData1_new}--\ref{fig:realData1_intel_new}. The contribution of the coreset is significant, already for smaller values of $n$, when it boosts other solvers that use cross validation for parameter tuning as explained above. In this case, the time complexity reduces by a factor of $m\cdot |\alphas|$ since the coreset is computed only once for each of the $m$ folds, regardless of the size $|\alphas|$. In practice, the running time is improved by a factor of x10--x100 as shown for example in Fig.~\ref{fig:ridgeD}--~\ref{fig:elasticD} and Fig.~\ref{fig:ridgeD_new}--~\ref{fig:elasticD_new}.
As shown in the graphs, the computations via Intel's Python distribution reduced the running times by 15-40\% compared to the default CPython distribution, with or without the booster. This is probably due to its tailored implementation for our hardware.

Furthermore, as expected, the running time of Algorithm~\ref{LMSCoresetnew} was faster than of Algorithm~\ref{LMSCoreset} when tuned appropriately, without much increase in the numerical error.

\paragraph{Numerical stability. }The SKETCH + CHOLESKY and SKETCH + SVD methods are simple and accurate in theory, and there is no hope to improve their running time via our much more involved booster. However, they are numerically unstable in practice for the reasons that are explained in Section~\ref{sec:relatedWork}. In fact, on most of our experiments we could not even apply the SKETCH + CHOLESKY technique at all using 32-bit floating point representation. This is because the resulting approximation to $A'^TA'$ was not a positive definite matrix as required by the Cholesky Decomposition, and we could not compute the matrix $L$ at all. In case of success, the running time of our algorithms was slower by at most a factor of $2$ but even in these cases numerical accuracy was improved up to orders of magnitude; See Fig.~\ref{fig:accuracy} and~\ref{fig:accuracy_new} for histogram of errors using such 32-bit float representation which is especially common in GPUs for saving memory, running time and power~\cite{wiki:tablecomp}. 
This is not surprising, even when considering our (potentially) less numerically accurate algorithm (Algorithm~\ref{LMSCoresetnew}). During its cumputation, Algorithm~\ref{LMSCoresetnew} simply sums over only $O(d^2)$ terms, where each is a sparse matrix, and then applies SVD, while the most numerically stable competing method \textbf{SKETCH + SVD} sums over $n$ non-sparse matrices and then applies SVD, which makes it less accurate, since the numerical error usually accumulates as we sum over more terms. 

For the special case of linear regression, we can apply SKETCH + INVERSE, which still has large numerical issues compared to our coreset computation as shown in Fig.~\ref{fig:accuracy} and~\ref{fig:accuracy_new}.

\begin{table}
\centering
\begin{adjustbox}{width=\textwidth}
\begin{tabular}{|c|c|l|c|c|l|l|}
\hline
Figure  & \makecell{Algorithm's\\number} & x/y Axes labels &  \makecell{Python Distribution} & Dataset & Input Parameter \\
\hline
\ref{fig:ridgeD},\ref{fig:lassoD},\ref{fig:elasticD} & \ref{RidgeOurAlg}--\ref{ElasticOurAlg} & Size/Time for various $d$ & CPython & Synthetic & $m=3$, $|\alphas|=100$\\
\hline
\ref{fig:ridge_alpha},\ref{fig:lasso_alpha},\ref{fig:elastic_alpha} & \ref{RidgeOurAlg}--\ref{ElasticOurAlg} & Size/Time for various $|\alphas|$ & CPython & Synthetic & $m=3, d=7$\\
\hline
\ref{fig:ridge_intel},\ref{fig:lasso_intel},\ref{fig:elastic_intel} & \ref{RidgeOurAlg}--\ref{ElasticOurAlg} & Size/Time for various $d$ & Intel's & Synthetic & $m=3, |\alphas|=100$\\
\hline
\ref{fig:ridge_alpha_intel},\ref{fig:lasso_alpha_intel},\ref{fig:elastic_alpha_intel} & \ref{RidgeOurAlg}--\ref{ElasticOurAlg} &Size/Time for various $|\alphas|$ & Intel's & Synthetic & $m=3, d=7$\\
\hline
\ref{fig:realData1},\ref{fig:realData2} & \ref{RidgeOurAlg}--\ref{ElasticOurAlg} & $|\alphas|$/Time & CPython & Datasets (i),(ii) & $m=3$\\
\hline
\ref{fig:realData1_intel},\ref{fig:realData2_intel} & \ref{RidgeOurAlg}--\ref{ElasticOurAlg} & $|\alphas|$/Time & Intel's & Datasets (i),(ii) & $m=3$\\
\hline
\ref{fig:3d_python_alphas},\ref{fig:house_python_alphas} & \ref{RidgeOurAlg}--\ref{ElasticOurAlg} & Time/maximal $|\alphas|$ that is feasible & CPython & Datasets (i),(ii) & $m=3$\\
\hline
\ref{fig:3d_intel_alphas},\ref{fig:house_intel_alphas} & \ref{RidgeOurAlg}--\ref{ElasticOurAlg} & Time/maximal $|\alphas|$ that is feasible & Intel's & Datasets (i),(ii) & $m=3$\\
\hline
\ref{fig:lstsq} & \ref{LstsqOurAlg} & Size/Time for various Distributions & CPython, Intel's & Synthetic & $m=64$, $d=15$\\
\hline
\ref{fig:accuracy} & \ref{LstsqOurAlg} & Error/Count Histogram + Size/Error & CPython & Datasets (i),(iii) & $m=1$\\
\hline
\hline
\ref{fig:ridgeD_new},\ref{fig:lassoD_new},\ref{fig:elasticD_new} & \ref{RidgeOurAlgnew}--\ref{ElasticOurAlgnew} & Size/Time for various $d$ & CPython & Synthetic & $m=3$, $|\alphas|=100$, $d'=12$\\
\hline
\ref{fig:ridge_alpha_new},\ref{fig:lasso_alpha_new},\ref{fig:elastic_alpha_new} & \ref{RidgeOurAlgnew}--\ref{ElasticOurAlgnew} & Size/Time for various $|\alphas|$ & CPython & Synthetic & $m=3, d=35$, $d'=12$\\
\hline
\ref{fig:ridge_intel_new},\ref{fig:lasso_intel_new},\ref{fig:elastic_intel_new} & \ref{RidgeOurAlgnew}--\ref{ElasticOurAlgnew} & Size/Time for various $d$ & Intel's & Synthetic & $m=3$, $|\alphas|=100$, $d'=12$\\
\hline
\ref{fig:ridge_alpha_intel_new},\ref{fig:lasso_alpha_intel_new},\ref{fig:elastic_alpha_intel_new} & \ref{RidgeOurAlgnew}--\ref{ElasticOurAlgnew} & Size/Time for various $|\alphas|$ & Intel's & Synthetic & $m=3, d=35$, $d'=12$\\
\hline
\ref{fig:realData1_new} & \ref{RidgeOurAlgnew}--\ref{ElasticOurAlgnew} & $|\alphas|$/Time & CPython & Dataset (iv) & $m=3$, $d'=17$\\
\hline
\ref{fig:realData1_intel_new} & \ref{RidgeOurAlgnew}--\ref{ElasticOurAlgnew} & $|\alphas|$/Time & Intel's & Dataset (iv) & $m=3$, $d'=17$\\
\hline
\ref{fig:accuracy_new} & \ref{LstsqOurAlg},\ref{LstsqOurAlgnew} & Error/Count Histogram + Size/Error & CPython & Datasets (iii) & $m=1$, $d'=12$\\
\hline
\end{tabular}
\end{adjustbox}
\caption{\textbf{Summary of experimental results}. CPython \protect\cite{wiki:cpython} and Intel's \protect\cite{inteldist} distributions were used. The input: $A\in\REAL^{n\times d}$ and $b\in\REAL^n$, where $n$ is ``Data size''. CV used $m$ folds for evaluating each parameter in $\alphas$. The chosen number of clusters in Algorithm~\ref{LMSCoreset} is $k=2(d+1)^2+2$.
The chosen parameters in Algorithm~\ref{LMSCoresetnew} were set to $k_2 = \ceil{(d+1)^2/d'}$ and $k_1 = 2d'+2$, where $d'$ is specified in the table.
The parameters $\rho=0.5$ was used for Algorithms~\ref{ElasticOurAlg} and~\ref{ElasticOurAlgnew}. 
Computation time includes the computation of the reduced input $(C,y)$; See Sections~\ref{sec:fastCara} and~\ref{sec:fastCaraNew}. The histograms consist of bins along with the number of errors that fall in each bin.} \label{expstable}
\end{table}

\newcommand\s{0.24}
\begin{figure*}[t!]
\centering
    \begin{subfigure}[t]{\s\textwidth}
		\centering
		\includegraphics[width = \textwidth]{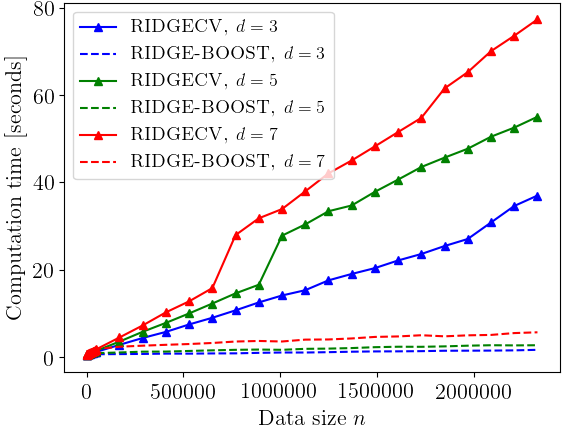}
        \caption{}
        \label{fig:ridgeD}
	\end{subfigure}
    \begin{subfigure}[t]{\s\textwidth}
		\centering
		\includegraphics[width = \textwidth]{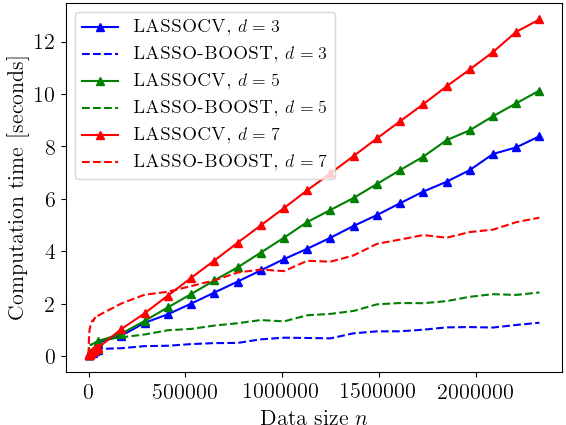}
        \caption{}
        \label{fig:lassoD}
	\end{subfigure}
    \begin{subfigure}[t]{\s\textwidth}
		\centering
		\includegraphics[width = \textwidth]{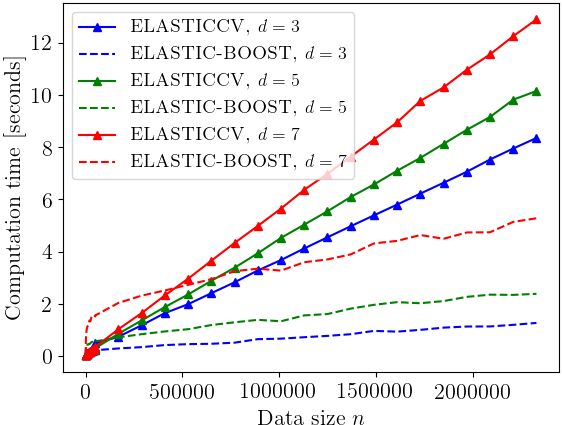}
        \caption{}
        \label{fig:elasticD}
	\end{subfigure}
    \begin{subfigure}[t]{\s\textwidth}
		\centering		
		\includegraphics[width = \textwidth]{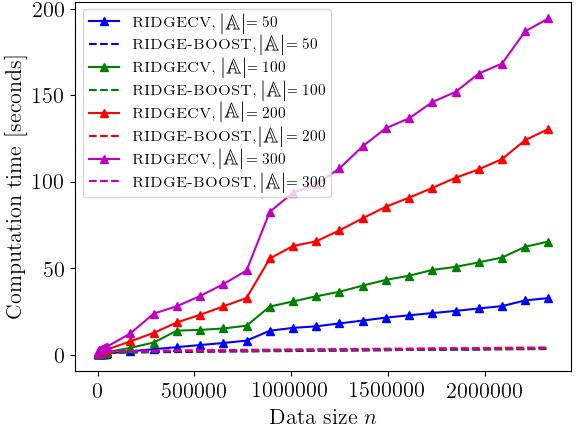}
        \caption{}
        \label{fig:ridge_alpha}
	\end{subfigure}
    \begin{subfigure}[t]{\s\textwidth}
		\centering		
		\includegraphics[width = \textwidth]{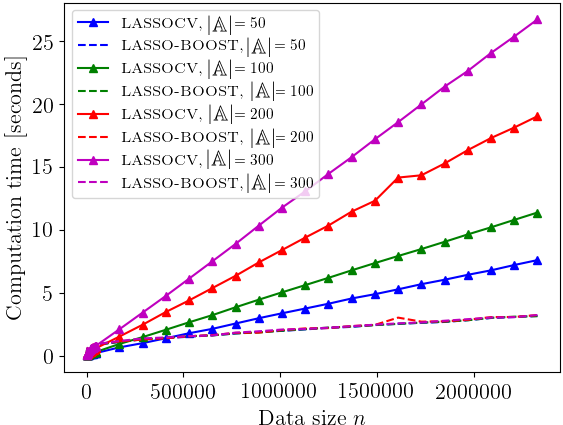}
        \caption{}
        \label{fig:lasso_alpha}
	\end{subfigure}
    \begin{subfigure}[t]{\s\textwidth}
		\centering		
		\includegraphics[width = \textwidth]{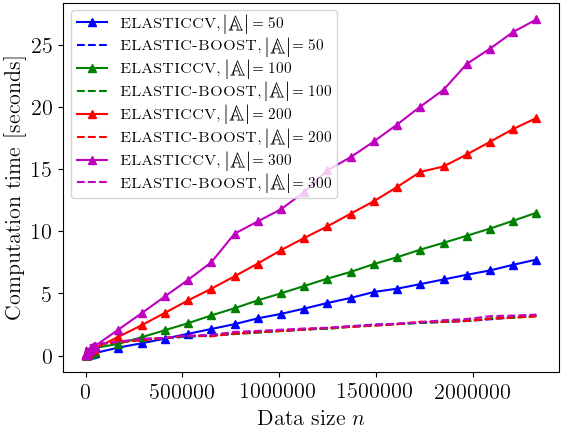}
        \caption{}
        \label{fig:elastic_alpha}
	\end{subfigure}
    \begin{subfigure}[t]{\s\textwidth}
		\centering		
		\includegraphics[width = \textwidth]{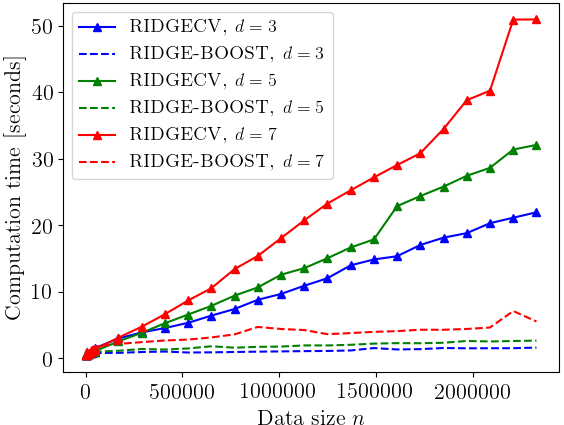}
        \caption{}
        \label{fig:ridge_intel}
	\end{subfigure}
    \begin{subfigure}[t]{\s\textwidth}
		\centering		
		\includegraphics[width = \textwidth]{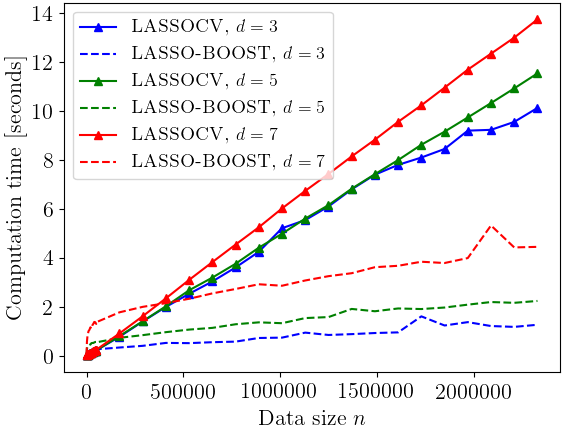}
        \caption{}
        \label{fig:lasso_intel}
	\end{subfigure}
    \begin{subfigure}[t]{\s\textwidth}
		\centering		
		\includegraphics[width = \textwidth]{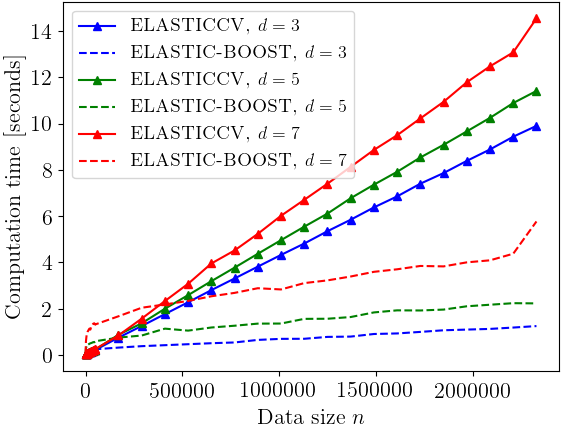}
        \caption{}
        \label{fig:elastic_intel}
	\end{subfigure}
    \begin{subfigure}[t]{\s\textwidth}
		\centering		
		\includegraphics[width = \textwidth]{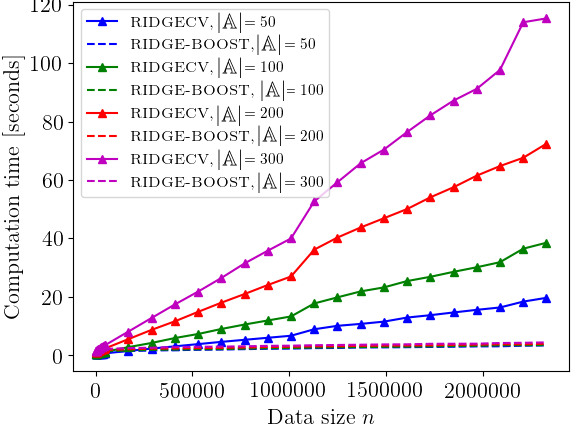}
        \caption{}
        \label{fig:ridge_alpha_intel}
	\end{subfigure}
    \begin{subfigure}[t]{\s\textwidth}
		\centering		
		\includegraphics[width = \textwidth]{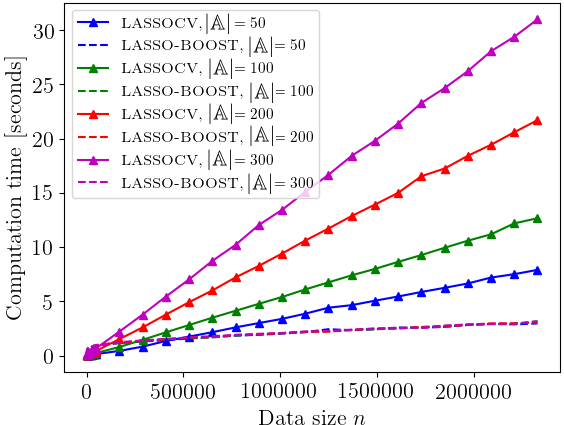}
        \caption{}
        \label{fig:lasso_alpha_intel}
	\end{subfigure}
    \begin{subfigure}[t]{\s\textwidth}
		\centering		
		\includegraphics[width = \textwidth]{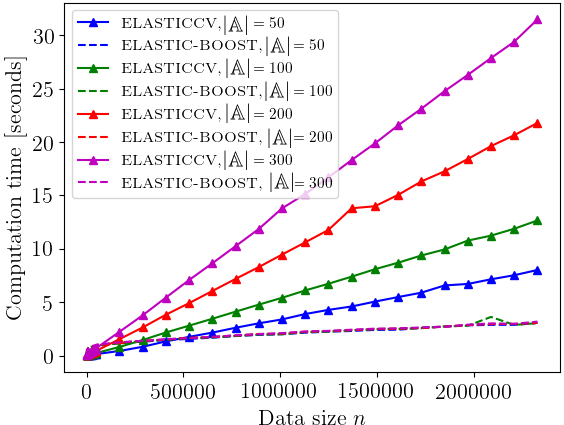}
        \caption{}
        \label{fig:elastic_alpha_intel}
	\end{subfigure}
    \begin{subfigure}[t]{\s\textwidth}
		\includegraphics[width = \textwidth]{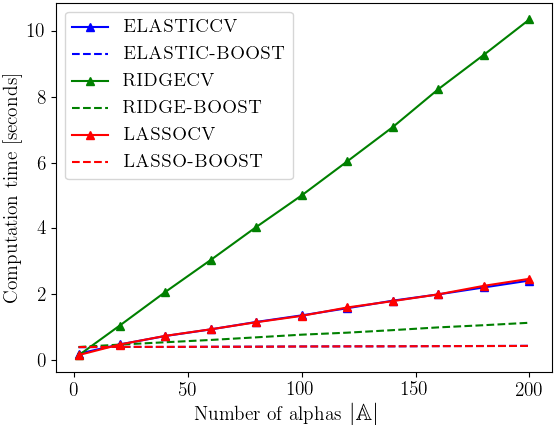}
        \caption{}
        \label{fig:realData1}
	\end{subfigure}
    \begin{subfigure}[t]{\s\textwidth}
		\includegraphics[width = \textwidth]{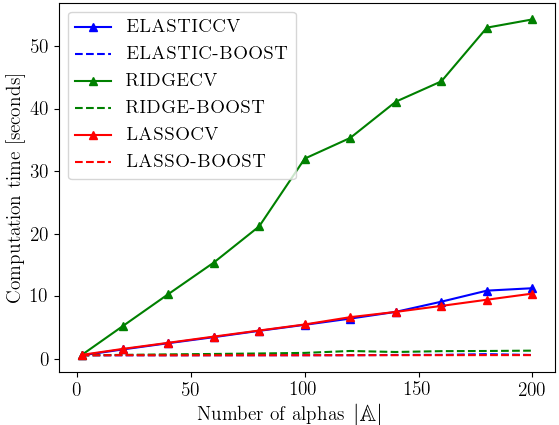}
        \caption{}
        \label{fig:realData2}
	\end{subfigure}
    \begin{subfigure}[t]{\s\textwidth}		
		\includegraphics[width = \textwidth]{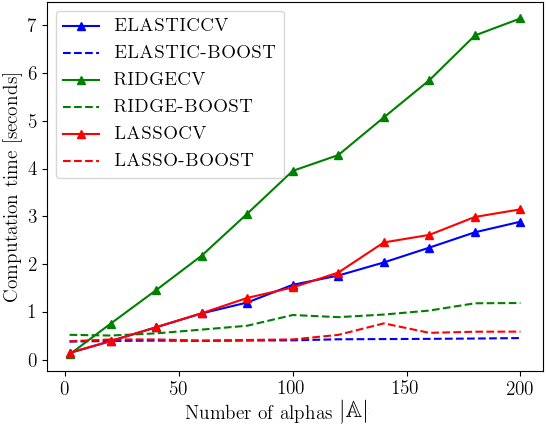}
        \caption{}
        \label{fig:realData1_intel}
	\end{subfigure}
    \begin{subfigure}[t]{\s\textwidth}	
		\includegraphics[width = \textwidth]{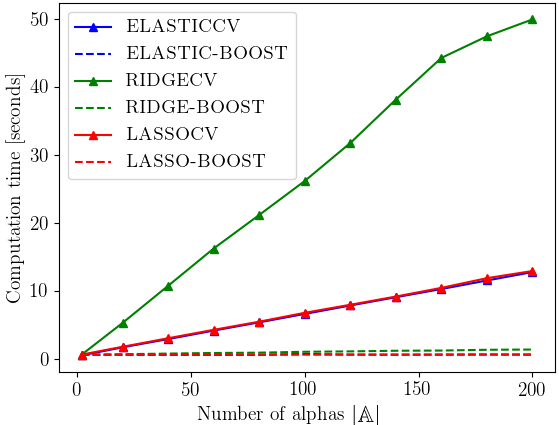}
        \caption{}
        \label{fig:realData2_intel}
	\end{subfigure}
    \begin{subfigure}[t]{\s\textwidth}		
		\includegraphics[width = \textwidth]{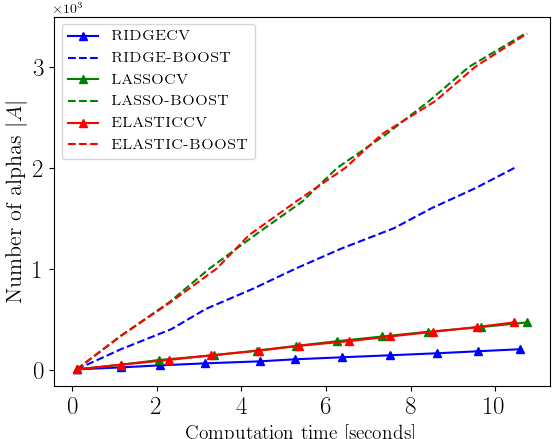}
        \caption{}
        \label{fig:3d_python_alphas}
	\end{subfigure}
    \begin{subfigure}[t]{\s\textwidth}	
		\includegraphics[width = \textwidth]{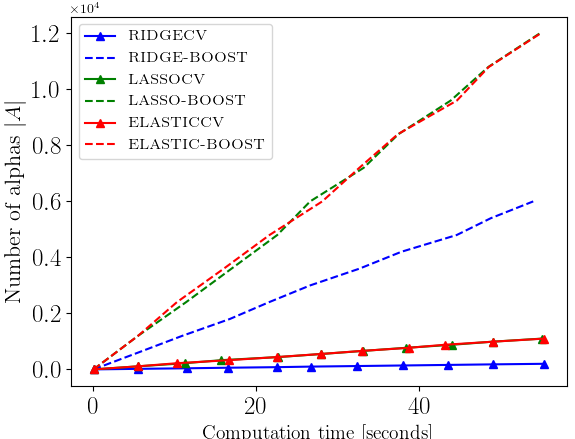}
        \caption{}
        \label{fig:house_python_alphas}
	\end{subfigure}
    \begin{subfigure}[t]{\s\textwidth}		
		\includegraphics[width = \textwidth]{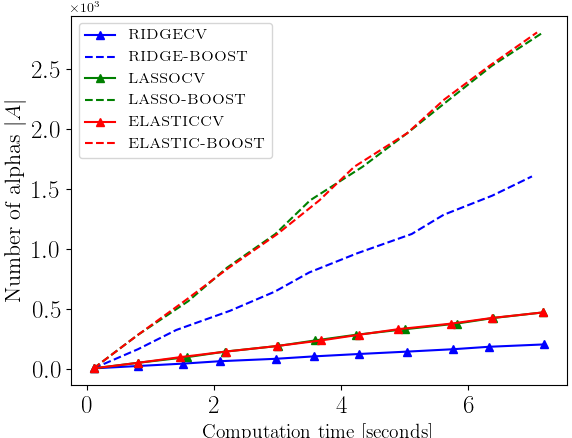}
        \caption{}
        \label{fig:3d_intel_alphas}
	\end{subfigure}
    \begin{subfigure}[t]{\s\textwidth}	
		\includegraphics[width = \textwidth]{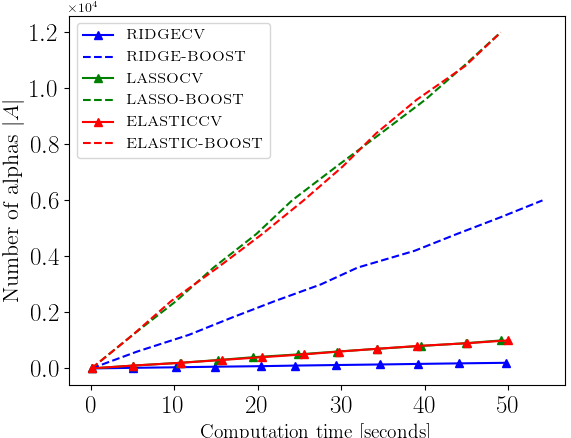}
        \caption{}
        \label{fig:house_intel_alphas}
	\end{subfigure}
    \begin{subfigure}[t]{\s\textwidth}	
		\includegraphics[width=\textwidth]{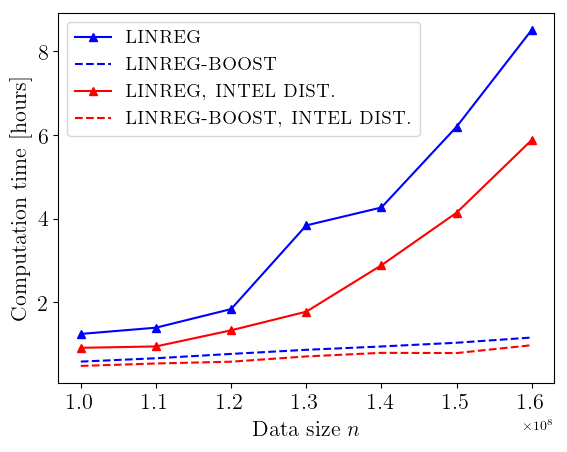}
        \caption{}
        \label{fig:lstsq}
	\end{subfigure}
    \begin{subfigure}[t]{0.74\textwidth}
		\centering	
		\includegraphics[width = 0.48\textwidth]{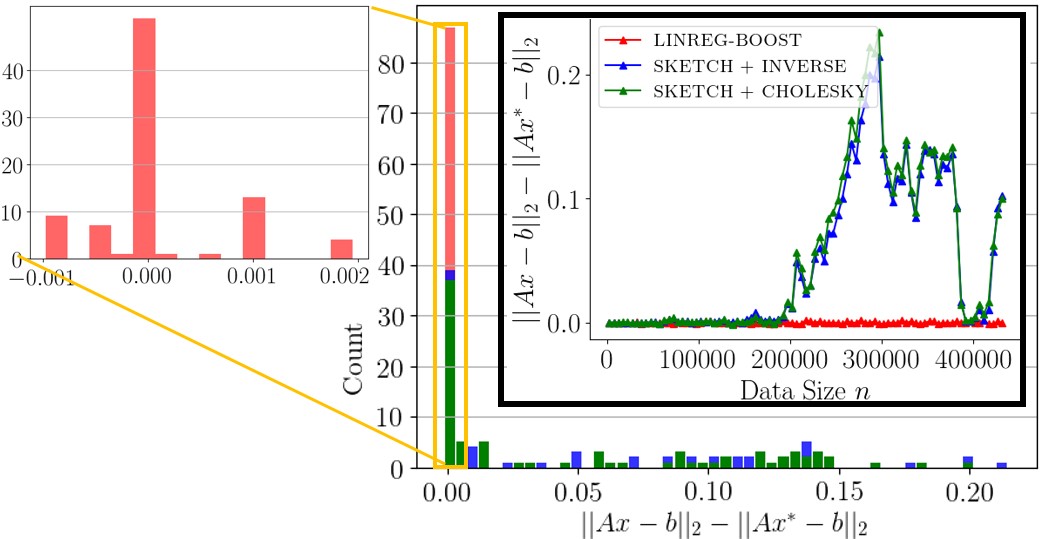}
		\includegraphics[width = 0.48\textwidth]{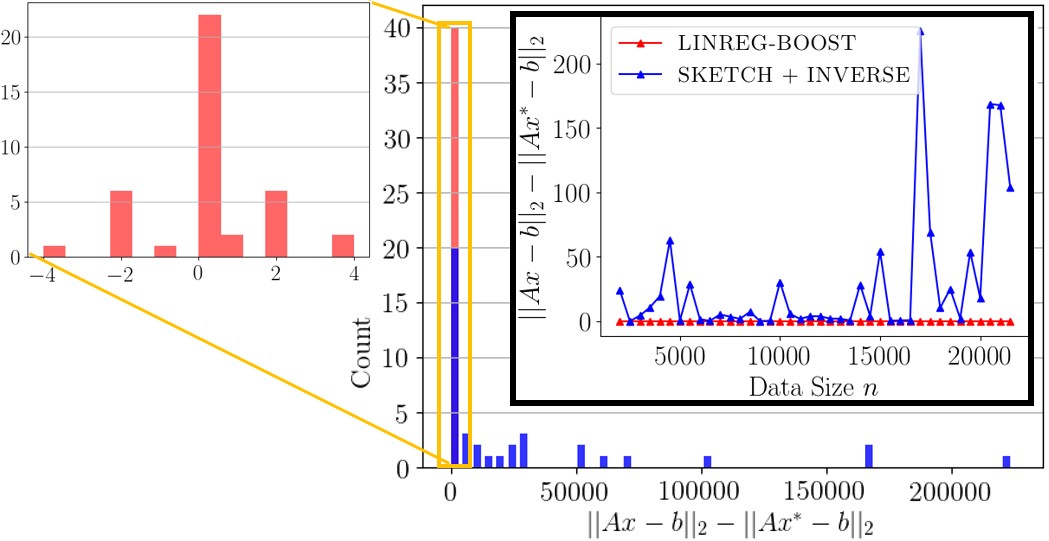}
        \caption{Accuracy comparison. (left): Dataset (i), (right): Dataset (ii). $x^* = \lstsqq(A,b)$. $x$ was computed using the methods specified in the legend; see Section~\ref{sec:discuss}.}
        \label{fig:accuracy}
	\end{subfigure}
    \caption{Experimental results; see Table~\ref{expstable}.}
    \label{fig:2}
\end{figure*}

\begin{figure*}[t!]
\centering
    \begin{subfigure}[t]{\s\textwidth}
		\centering
		\includegraphics[width = \textwidth]{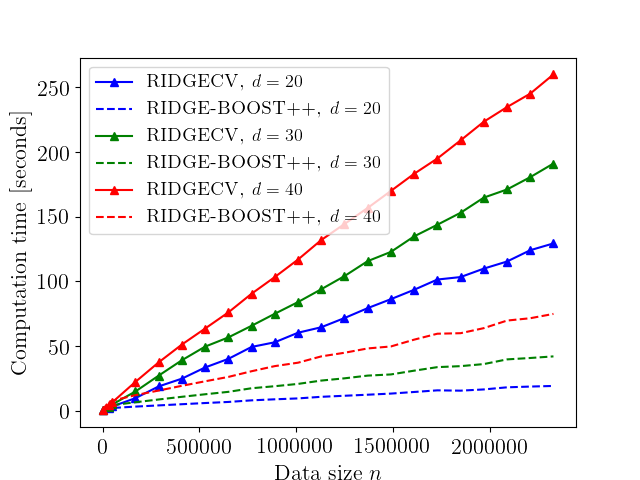}
        \caption{}
        \label{fig:ridgeD_new}
	\end{subfigure}
    \begin{subfigure}[t]{\s\textwidth}
		\centering
		\includegraphics[width = \textwidth]{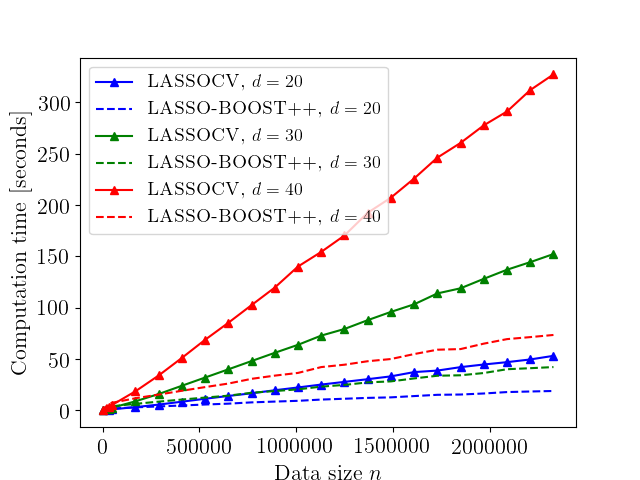}
        \caption{}
        \label{fig:lassoD_new}
	\end{subfigure}
    \begin{subfigure}[t]{\s\textwidth}
		\centering
		\includegraphics[width = \textwidth]{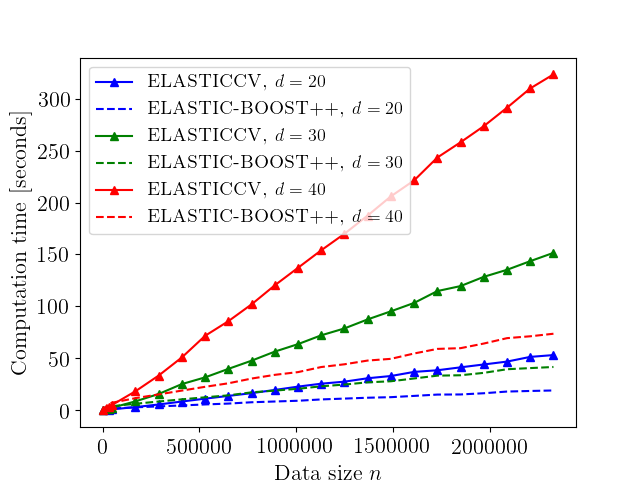}
        \caption{}
        \label{fig:elasticD_new}
	\end{subfigure}
    \begin{subfigure}[t]{\s\textwidth}
		\centering		
		\includegraphics[width = \textwidth]{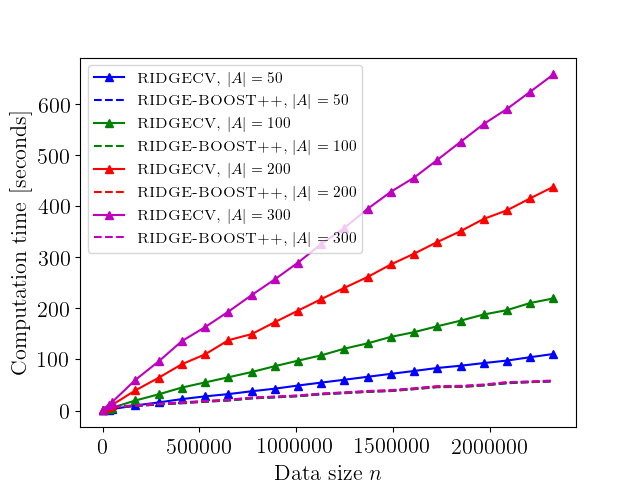}
        \caption{}
        \label{fig:ridge_alpha_new}
	\end{subfigure}
    \begin{subfigure}[t]{\s\textwidth}
		\centering		
		\includegraphics[width = \textwidth]{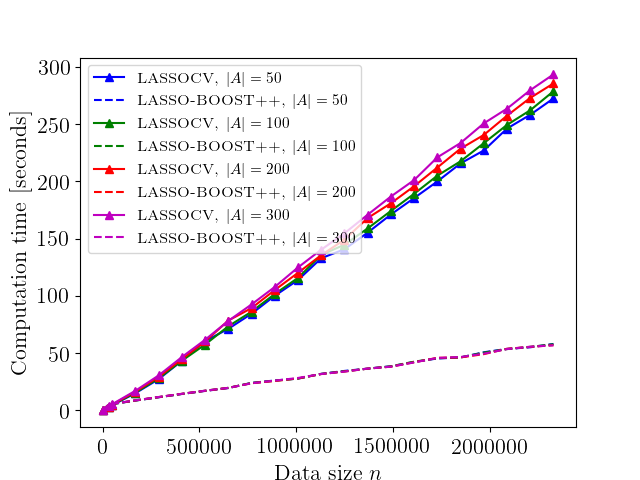}
        \caption{}
        \label{fig:lasso_alpha_new}
	\end{subfigure}
    \begin{subfigure}[t]{\s\textwidth}
		\centering		
		\includegraphics[width = \textwidth]{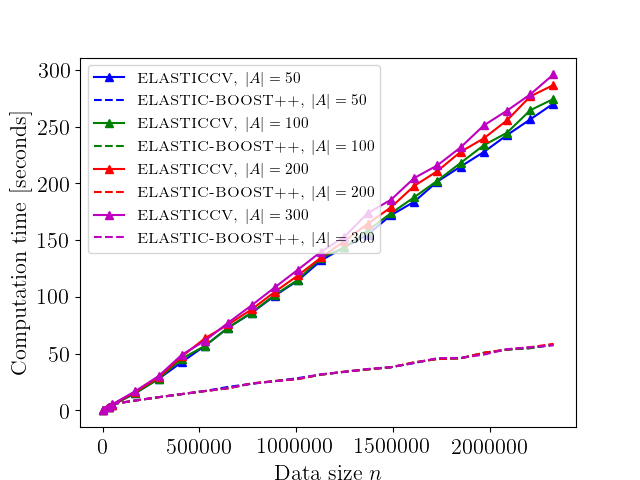}
        \caption{}
        \label{fig:elastic_alpha_new}
	\end{subfigure}
    \begin{subfigure}[t]{\s\textwidth}
		\centering		
		\includegraphics[width = \textwidth]{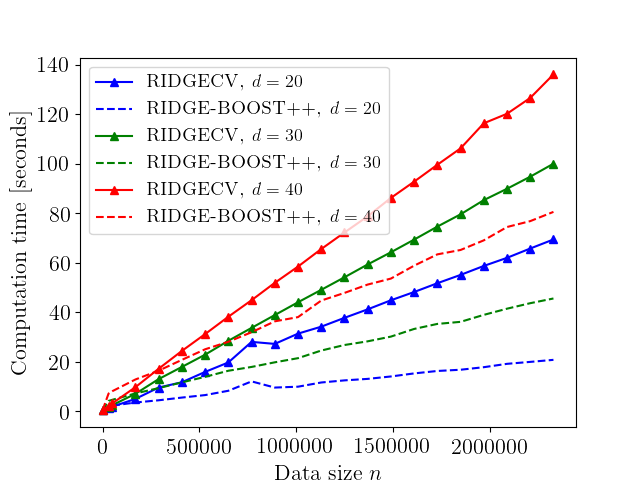}
        \caption{}
        \label{fig:ridge_intel_new}
	\end{subfigure}
    \begin{subfigure}[t]{\s\textwidth}
		\centering		
		\includegraphics[width = \textwidth]{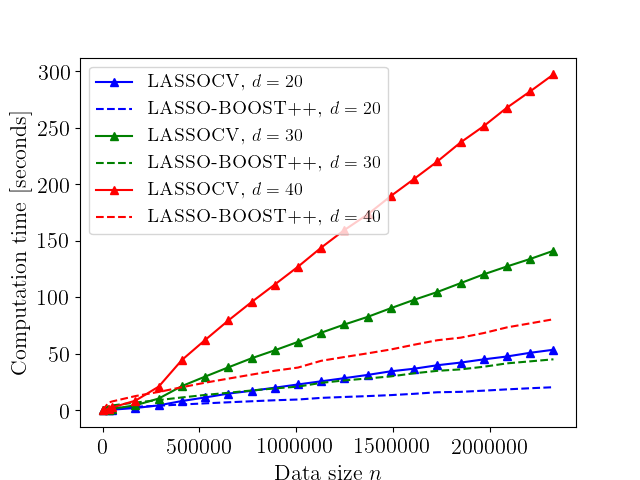}
        \caption{}
        \label{fig:lasso_intel_new}
	\end{subfigure}
    \begin{subfigure}[t]{\s\textwidth}
		\centering		
		\includegraphics[width = \textwidth]{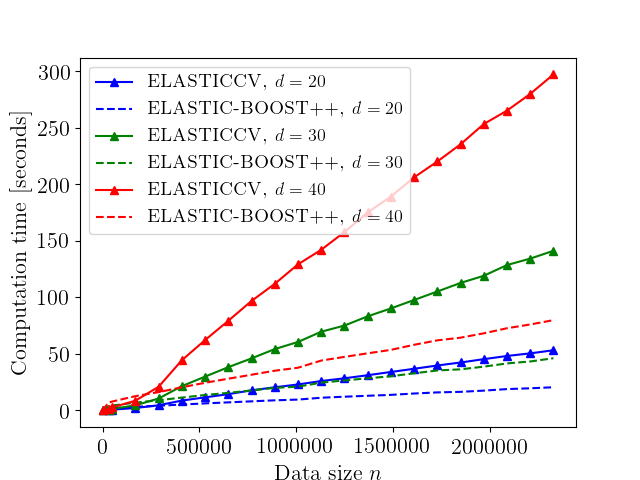}
        \caption{}
        \label{fig:elastic_intel_new}
	\end{subfigure}
    \begin{subfigure}[t]{\s\textwidth}
		\centering		
		\includegraphics[width = \textwidth]{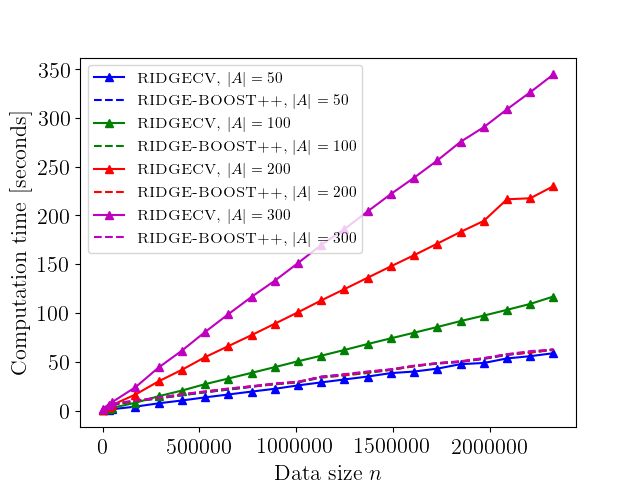}
        \caption{}
        \label{fig:ridge_alpha_intel_new}
	\end{subfigure}
    \begin{subfigure}[t]{\s\textwidth}
		\centering		
		\includegraphics[width = \textwidth]{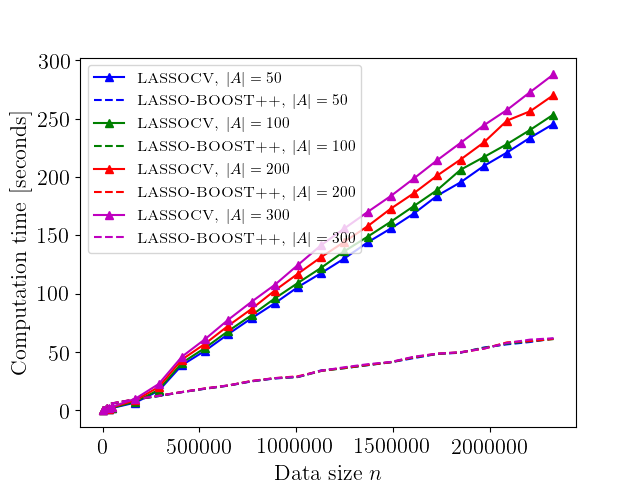}
        \caption{}
        \label{fig:lasso_alpha_intel_new}
	\end{subfigure}
    \begin{subfigure}[t]{\s\textwidth}
		\centering		
		\includegraphics[width = \textwidth]{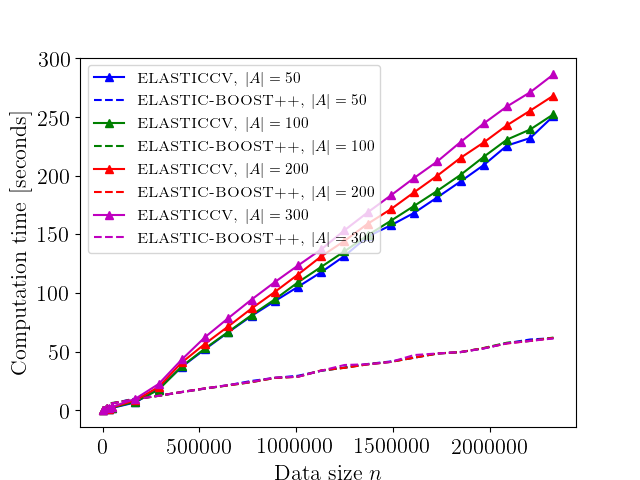}
        \caption{}
        \label{fig:elastic_alpha_intel_new}
	\end{subfigure}
    \begin{subfigure}[t]{\s\textwidth}
		\includegraphics[width = \textwidth]{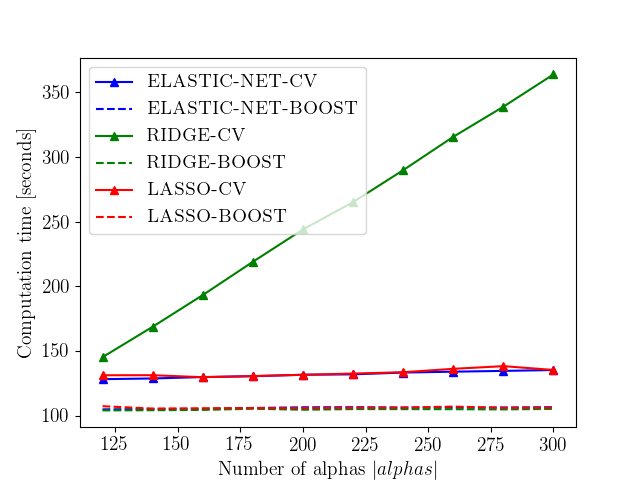}
        \caption{}
        \label{fig:realData1_new}
	\end{subfigure}
    \begin{subfigure}[t]{\s\textwidth}		
		\includegraphics[width = \textwidth]{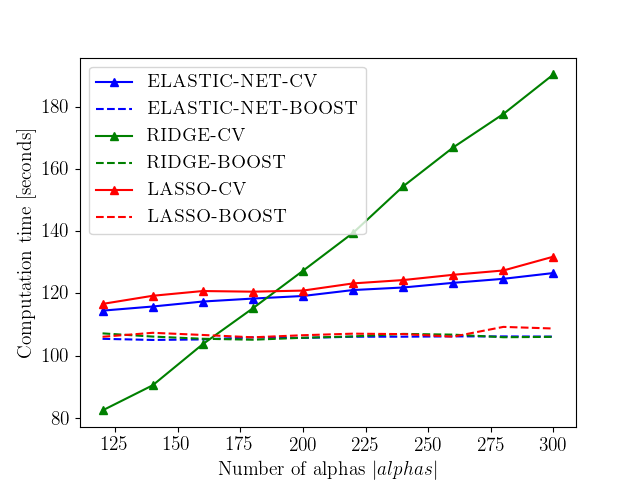}
        \caption{}
        \label{fig:realData1_intel_new}
	\end{subfigure}
    \begin{subfigure}[t]{0.5\textwidth}
		\centering
		\includegraphics[width = 0.5\textwidth]{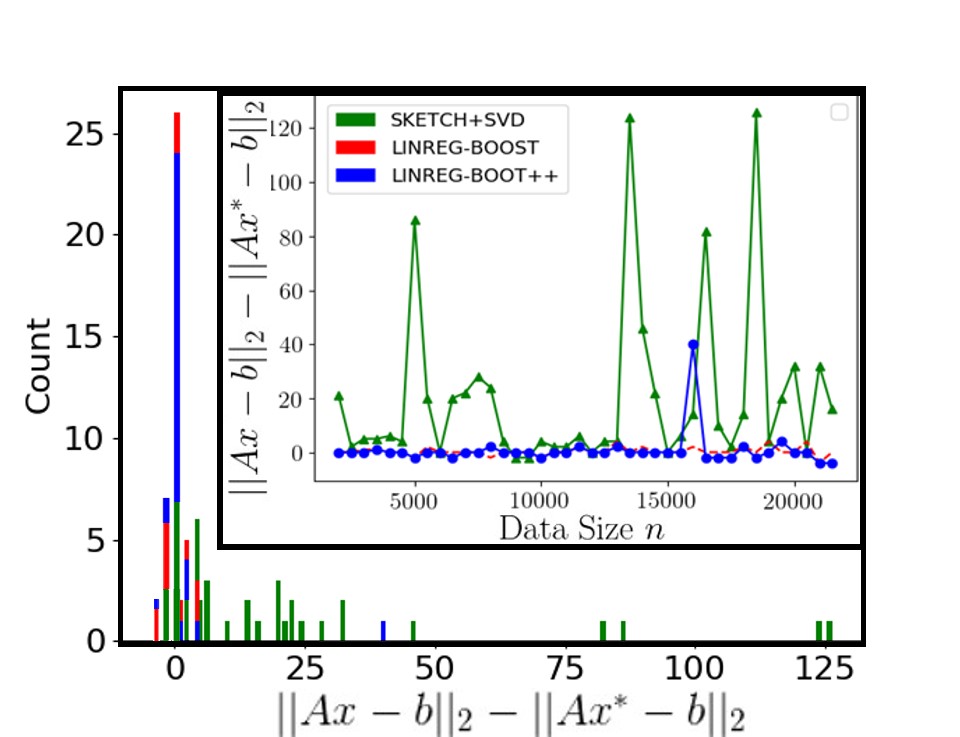}
        \caption{Accuracy comparison. (left): Dataset (i), (right): Dataset (ii). $x^* = \lstsqq(A,b)$. $x$ was computed using the methods specified in the legend; see Section~\ref{sec:discuss}.}
        \label{fig:accuracy_new}
	\end{subfigure}
    \caption{Experimental results; see Table~\ref{expstable}.}
    \label{fig:3}
\end{figure*}

\section{Conclusion and Future Work} \label{sec:conclude}
We presented a novel framework that combines sketches and coresets.
As an example application, we proved that the set from the Caratheodory Theorem can be computed in $O(nd)$ overall time for sufficiently large $n$ instead of the $O(n^2d^2)$ time as in the original theorem. We then generalized the result for a matrix $S$ whose rows are a weighted subset of the input matrix and their covariance matrix is the same. Our experimental results section shows how to significantly boost the numerical stability or running time of existing LMS solvers by applying them on $S$.
Future work includes: (a) applications of our framework to combine other sketch-coreset pairs e.g. as listed in~\cite{phillips2016coresets}, (b) Experiments for streaming/distributed/GPU data, and (c) generalization of our approach for more complicated models and applications, e.g., deep learning, decision trees, and many more.

\clearpage

\vskip 0.2in
\bibliography{main_ref}

\newpage
\appendix

\section{Slow Caratheodory Implementation} \label{sec:CaraAlg}

\setcounter{AlgoLine}{0}
\begin{algorithm}[th]
\caption{$\caras(P,u)$}\label{caraAlg}
\SetKwInOut{Input}{Input}
\SetKwInOut{Output}{Output}
\Input{A weighted set $(P,u)$ of $n$ points in $\REAL^{d}$.}
\Output{A Caratheodory set $(S,w)$ for $(P,u)$ in $O(n^2d^2)$ time.}

\If{$n\leq d+1$}
{\Return $(P,u)$}

Identify $P = \br{p_1,\cdots,p_n}$

\For {every $i\in\br{2,\cdots,n}$}
{
$a_i := p_i - p_1$
}

$A := ( a_2 \mid \cdots \mid a_{n})$ \tcp{$A\in \REAL^{d\times (n-1)}$}

Compute $v=(v_2,\cdots,v_{n})^T\neq 0$ such that $Av=0$. \label{l11}\\

$\displaystyle v_1 := -\sum_{i=2}^{n}  v_i$\label{u1}\\

$\displaystyle \alpha := \min\br{\frac{u_i}{v_i } \mid i\in \br{1,\cdots,n} \text{ and }  v_i> 0}$ \label{alp}\\

$w_i := u_i-\alpha v_i$ for every $i\in\br{1,\cdots,n}$. \label{Sdef}\\

$S := \br{p_i\mid w_i>0 \text{ and } i\in\br{1,\cdots,n}}$\label{Sdeff}\\
\If {$|S|>d+1$}{
$(S,w):= \caras(S,w)$ \label{eight}}
\Return $(S,w)$
\end{algorithm}

\textbf{Overview of Algorithm~\ref{caraAlg} and its correctness.}
The input is a weighted set $(P,u)$ whose points are denoted by $P=\br{p_1,\cdots,p_n}$. We assume $n>d+1$, otherwise $(S,w)=(P,u)$ is the desired coreset. Hence, the $n-1>d$ points $p_2-p_1$, $p_3-p_1,\ldots,p_n-p_1 \in \REAL^d$ must be linearly dependent. This implies that there are reals $v_2,\cdots,v_n$, which are not all zeros, such that
\begin{equation}\label{eq00}
\sum_{i=2}^n v_i (p_i-p_1)=0.
\end{equation}
These reals are computed in Line~\ref{l11} by solving system of linear equations. This step dominates the running time of the algorithm and takes $O(nd^2)$ time using e.g. SVD. The definition
\begin{equation}\label{uudef}
v_1=-\sum_{i=2}^n v_i
\end{equation}
in Line~\ref{u1}, guarantees that
\begin{equation} \label{negativev}
v_j < 0 \text{ for some }j\in [n],
\end{equation}
and that
\begin{equation}\label{eqq}
\begin{split}
\sum_{i=1}^n v_i p_i&=v_1p_1+\sum_{i=2}^n v_i p_i=\left(-\sum_{i=2}^n v_i\right) p_1+\sum_{i=2}^n v_i p_i
=\sum_{i=2}^n v_i (p_i-p_1)=0,
\end{split}
\end{equation}
where the second equality is by~\eqref{uudef}, and the last is by~\eqref{eq00}.
Hence, for every $\alpha\in\REAL$, the weighted mean of $P$ is
\begin{equation}\label{sum}
\sum_{i=1}^n u_ip_i=\sum_{i=1}^n u_ip_i-\alpha\sum_{i=1}^n v_i p_i=  \sum_{i=1}^n \left(u_i-\alpha v_i\right) p_i,
\end{equation}
where the first equality holds since $\sum_{i=1}^n v_i p_i = 0$ by~\eqref{eqq}.
The definition of $\alpha$ in Line~\ref{alp} guarantees that $\alpha v_{i^*}=u_{i^*}$ for some $i^*\in[n]$, and that $u_i-\alpha v_i\geq 0$ for every $i\in[n]$. Hence, the set $S$ that is defined in Line~\ref{Sdeff} contains at most $n-1$ points, and its set of weights $\br{u_i-\alpha v_i}$ is non-negative. Notice that if $\alpha = 0$, we have that $w_j = u_j > 0$ for some $j\in [n]$. Otherwise, if $\alpha > 0$, by~\eqref{negativev} there is $j\in [n]$ such that $v_j < 0$, which yields that $w_j = u_j -\alpha v_j > 0$. Hence, in both cases there is $w_j > 0$ for some $j\in [n]$. Therefore, $|S| \neq \emptyset$.

The sum of the positive weights is thus the total sum of weights,
\[
\sum_{p_i\in S}^n w_i=\sum_{i=1}^n (u_i-\alpha v_i)=\sum_{i=1}^nu_i-\alpha\cdot \sum_{i=1}^n v_i=1,
\]
where the last equality hold by~\eqref{uudef}, and since $u$ sums to $1$. This and~\eqref{sum} proves that $(S,w)$ is a Caratheodory set of size $n-1$ for $(P,u)$; see Definition~\ref{def:caraSet}. In Line~\ref{eight} we repeat this process recursively until there are at most $d+1$ points left in $S$. For $O(n)$ iterations, the overall time is thus $O(n^2d^2)$.

\section{Faster Caratheodory Set}

\begin{theorem}[Theorem~\ref{theorem:fastCara}] \label{theorem:fastCara_proof}
Let $(P,u)$ be a weighted set of $n$ points in $\REAL^d$ such that $\sum_{p\in P}u(p)=1$, and $k\geq d+2$ be an integer.
Let $(C,w)$ be the output of a call to $\caraf(P,u,k)$; See Algorithm~\ref{fastCara}.
Let $\tcara(k,d)$ be the time it takes to compute a Caratheodory Set for $k$ points in $\REAL^d$, as in Theorem~\ref{caraTheorem}. Then $(C,w)$ is a Caratheodory set of $(P,u)$ that is computed in time
\[
O\left(nd+\tcara(k,d)\cdot \frac{\log n}{\log (k/ d)}\right).
\]
\end{theorem}
\begin{proof}
We use the notation and variable names as defined in Algorithm~\ref{fastCara} from Section~\ref{sec:fastCara}.

First, at Line~\ref{Line:removeZero} we remove all the points in $P$ which have zero weight, since they do not contribute to the weighted sum. Therefore, we now assume that $u(p) >0$ for every $p\in P$ and that $|P| = n$.
Identify the input set $P =\br{p_1,\cdots,p_n}$ and the set $C$ that is computed at Line~\ref{compC} of Algorithm~\ref{fastCara} as $C = \br{c_1,\cdots,c_{|C|}}$.
We will first prove that the weighted set $(C,w)$ that is computed in Lines~\ref{compC}--\ref{compW} at an arbitrary iteration is a Caratheodory set for $(P,u)$, i.e., $C \subseteq P$, $\sum_{p\in P}u(p)\cdot p=\sum_{p\in C}w(p)\cdot p$, $\sum_{p\in P}u(p)=\sum_{p\in C}w(p)$ and $|C| \leq (d+1)\cdot\ceil{\frac{n}{k}}$.

Let $(\tilde{\mu},\tilde{w})$ be the pair that is computed during the execution the current iteration at Line~\ref{compSlowCara}.
By Theorem~\ref{caraTheorem} and Algorithm~\ref{caraAlg}, the pair $(\tilde{\mu},\tilde{w})$ is a Caratheodory set of the weighted set $(\br{\mu_1,\cdots,\mu_k},u')$. Hence,
\begin{equation}\label{eqProps}
\sum_{\mu_i \in \tilde{\mu}} \tilde{w}(\mu_i) = 1, \quad \sum_{\mu_i \in \tilde{\mu}} \tilde{w}(\mu_i) \mu_i = \sum_{i=1}^{k} u'(\mu_i) \cdot \mu_i, \quad \tilde{\mu}\subseteq \br{\mu_1,\cdots,\mu_k} \text{ and } \quad |\tilde{\mu}| \leq d+1.
\end{equation}

By the definition of $\mu_i$, for every $i \in \br{1,\cdots,k}$
\begin{equation}\label{eqDefMui}
\sum_{i=1}^{k} u'(\mu_i)\cdot \mu_i = \sum_{i=1}^{k} u'(\mu_i)\cdot \left(\frac{1}{u'(\mu_i)} \cdot \sum_{p\in P_i}u(p)\cdot p\right) = \sum_{i=1}^{k} \sum_{p\in P_i} u(p)p = \sum_{p\in P} u(p) p.
\end{equation}
By Line~\ref{compC} we have that
\begin{equation} \label{eq:CisSubset}
C \subseteq P.
\end{equation}
We also have that
\begin{equation} \label{sameWeightedSum}
\begin{split}
\sum_{p\in C} w(p)p &
= \sum_{\mu_i \in \tilde{\mu}} \sum_{p\in P_i}\frac{\tilde{w}(\mu_i)u(p)}{u'(\mu_i)} \cdot  p
= \sum_{\mu_i \in \tilde{\mu}}\tilde{w}(\mu_i) \sum_{p\in P_i} \frac{u(p)}{u'(\mu_i)} p
= \sum_{\mu_i \in \tilde{\mu}} \tilde{w}(\mu_i) \mu_i\\
& = \sum_{i=1}^{k} u'(\mu_i)\cdot \mu_i
= \sum_{p\in P} u(p) p,
\end{split}
\end{equation}
where the first equality holds by the definitions of $C$ and $w$, the third equality holds by the definition of $\mu_i$ at Line~\ref{compMui}, the fourth equality is by~\eqref{eqProps}, and the last equality is by~\eqref{eqDefMui}.

The new sum of weights is equal to
\begin{equation}\label{eqSumNewWeights}
\sum_{p\in C} w(p) = \sum_{\mu_i \in \tilde{\mu}} \sum_{p\in P_i} \frac{\tilde{w}(\mu_i)u(p)}{u'(\mu_i)} = \sum_{\mu_i \in \tilde{\mu}} \frac{\tilde{w}(\mu_i)}{u'(\mu_i)} \cdot \sum_{p\in P_i} u(p) = \sum_{\mu_i \in \tilde{\mu}} \frac{\tilde{w}(\mu_i)}{u'(\mu_i)} \cdot u'(\mu_i) = \sum_{\mu_i \in \tilde{\mu}} \tilde{w}(\mu_i) = 1,
\end{equation}
where the last equality is by~\eqref{eqProps}.

Combining~\eqref{eq:CisSubset},~\eqref{sameWeightedSum} and~\eqref{eqSumNewWeights} yields that the weighted $(C,w)$ computed before the recursive call at Line~\ref{recursiveCall} of the algorithm is a Caratheodory set for the weighted input set $(P,u)$.
Since at each iteration we either return such a Caratheodory set $(C,w)$ at Line~\ref{recursiveCall} or return the input weighted set $(P,u)$ itself at Line~\ref{returnInput}, by induction we conclude that the output weighted set of a call to $\caraf(P,u,k)$ is a Caratheodory set for the original input $(P,u)$.

By~\eqref{eqProps} we have that $C$ contains at most $(d+1)$ clusters from $P$ and at most $|C| \leq (d+1)\cdot\ceil{\frac{n}{k}} $ points. Hence, there are at most $\log_{\frac{k}{d+1}}(n)$ recursive calls before the stopping condition in line~\ref{stoprule} is satisfied. The time complexity of each iteration is $n'+\tcara(k,d)$ where $n' = |P|\cdot d$ is the number of points in the current iteration. Thus the total running time of Algorithm~\ref{fastCara} is
\[
\sum_{i=1}^{\log_{\frac{k}{d+1}}(n)} \left(\frac{nd}{2^{i-1}}+ \tcara(k,d)\right) \leq 2nd + \log_{\frac{k}{d+1}}(n) \cdot \tcara(k,d) \in O\left(nd+ \frac{\log{n}}{\log(k/(d+1))} \cdot \tcara(k,d)\right).
\]
\end{proof}

\begin{theorem} [Theorem~\ref{theorem:cov}] \label{theorem:cov_proof}
Let $A \in \REAL^{n\times d}$ be a matrix, and $k \geq d^2+2$ be an integer. Let $S \in \REAL^{(d^2+1)\times d}$ be the output of a call to $\cova(A,k)$; see Algorithm~\ref{covAlg}. Let $\tcara(k,d)$ be the computation time of \caras{} given $k$ point in $\REAL^{d^2}$. Then $S$ satisfies that $A^TA = S^TS$. Furthermore, $S$ can be computed in $O(nd^2+\tcara(k,d^2)\cdot\frac{\log{n}}{\log{(k/d^2))}})$ time.
\end{theorem}
\begin{proof}
We use the notation and variable names as defined in Algorithm~\ref{covAlg} from Section~\ref{sec:fastCara}.

Since $(C,w)$ at Line~\ref{compCU} of Algorithm~\ref{covAlg} is the output of a call to $\caraf(P,u,k)$, by Theorem~\ref{theorem:fastCara} we have that: (i) the weighted means of $(C,w)$ and $(P,u)$ are equal, i.e.,
\begin{equation} \label{eqSameWeightedMean}
\sum_{p\in P}u(p)\cdot p=\sum_{p\in C}w(p)\cdot p,
\end{equation}
(ii) $|C| \leq d^2+1$ since $P \subseteq \REAL^{(d^2)}$, and (iii) $C$ is computed in $O(nd^2+\log_{\frac{k}{d^2+1}}(n)\cdot \tcara(k,d^2))$ time.

Combining~\eqref{eqSameWeightedMean} with the fact that $p_i$ is simply the concatenation of the entries of $a_ia_i^T$, we have that
\begin{equation}\label{eqSameCov}
\sum_{p_i \in P} u(p_i) a_ia_i^T = \sum_{p_i \in C} w(p_i)\cdot a_ia_i^T.
\end{equation}
By the definition of $S$ in Line~\ref{compS}, we have that
\begin{equation}\label{eqSTS}
S^TS = \sum_{p_i \in C} (\sqrt{n\cdot w(p_i)}\cdot a_i)(\sqrt{n\cdot w(p_i)}\cdot a_i)^T = n\cdot \sum_{p_i \in C} w(p_i)\cdot a_ia_i^T.
\end{equation}
We also have that
\begin{equation}\label{eqATA}
A^TA = \sum_{i=1}^n a_ia_i^T = n \cdot \sum_{p_i \in P}(1/n) a_ia_i^T = n\cdot \sum_{p_i \in P} u(p_i) a_ia_i^T,
\end{equation}
where the second derivation holds since $u \equiv 1/n$. Theorem~\ref{theorem:cov} now holds by combining~\eqref{eqSameCov},~\eqref{eqSTS} and~\eqref{eqATA} as
\[
S^TS = n\cdot \sum_{p_i \in C} w(p_i)\cdot a_ia_i^T = n \cdot \sum_{p_i \in P} u(p_i) a_ia_i^T  = A^TA.
\]

\noindent\textbf{Running time: } Computing the weighted set $(P,u)$ at Lines~\ref{forPcreate}--~\ref{forPcreateend} takes $O(nd^2)$ time, since it takes $O(d^2)$ time to compute each of the $n$ points in $P$.

By Theorem~\ref{theorem:fastCara}, Line~\ref{compCU} takes $O(nd^2+\tcara(k,d^2)\cdot\frac{\log{n}}{\log{(k/d^2)}})$ to compute a \caras{} for the the weighted set $(P,u)$, and finally Line~\ref{compS} takes $O(d^3)$ for building the matrix $S$. Hence, the overall running time of Algorithm~\ref{covAlg} is $O(nd^2+\tcara(k,d^2)\cdot\frac{\log{n}}{\log{(k/d^2)}})$.
\end{proof}

\section{Sparsified Caratheodory}
\begin{theorem}\label{theorem:cordCara2}
Let $(P,u)$ be a weighted set of $n$ points in $\REAL^d$ such that $\sum_{p\in P}u(p)=1$, and $k_1,k_2,d'$ be three integers such that $k_2 \in \br{1,\cdots,d}$, $d' = \ceil{\frac{d}{k_2}}$, and $k_1 \in \br{d'+2,\cdots,n}$.
Let $(C,w)$ be the output of a call to $\caracord(P,u,k_1,k_2)$; See Algorithm~\ref{cordAlg}.
Let $\tcara(k_1,d')$ be the time it takes to compute a Caratheodory Set for $k_1$ points in $\REAL^{d'}$, as in Theorem~\ref{caraTheorem}. Then $(C,w)$ is a $d'$-Sparse Caratheodory set of $(P,u)$ that is computed in time $O\left(nd+\tcara(k_1,d')\cdot \frac{k_2\log n}{\log (k_1/d')}\right)$.
\end{theorem}
\begin{proof}
We consider the variables from Algorithm~\ref{cordAlg}. 
At Line~\ref{Line:partitionCord} we define a partition $I_1,\cdots,I_{k_2}$ of the coordinates (indices) into $k_2$ (almost) equal sized subsets, each of size at most $\ceil{\frac{d}{k_2}}$. 

Put $j\in [k_2]$. At Line~\ref{Line:compPj}, we compute the set $P^j$ that contains the entire input points, where each point is restricted to only a subset of its coordinates whose indices are in $I_j$. Each new point $p^j \in P^j \subseteq \REAL^{|I_j|}$, that contains a subset of the coordinates of some original point $p \in P$, is assigned a weight $u^j(p^j)$ that is equal to the original weight $u(p)$ of $p$ at Line~\ref{Line:compuj}.
In other words, the weighted set $(P^j,u^j)$ is basically a restriction of the input $(P,u)$ to a subset of the coordinates.

By Theorem~\ref{theorem:fastCara}, the weighted set $(C^j,w^j) := \caraf(P^j,u^j,k_1)$ computed at Line~\ref{Line:callBoost} via a call to Algorithm~\ref{fastCara} is thus a Caratheodory set of $(P^j,u^j)$, where $|C^j| \leq |I_j|+1 = d'+1$. Therefore,
\begin{equation} \label{eq:sumWc}
\sum_{c \in \hat{C}^j} w(c) = \sum_{c \in C^j} w^j(c) = \sum_{p\in P^j}u^j(p) = \sum_{p\in P}u(p) = 1,
\end{equation}
and
\begin{equation} \label{eq:equalSumCords}
\sum_{c \in C^j} w^j(c) c = \sum_{p\in P^j}u^j(p) p.
\end{equation}

Then, at Lines~\ref{Line:padVecs}--\ref{Line:compw}, we plug every $c\in C^j$ into a $d$-dimensional zeros vector $\hat{c}$ in the coordinates contained in $I_j$, and assign this new vector the same weight $w(\hat{c}) = w^j(c)$ of $c$. Combining that the weighted sum of $(P^j,u^j)$, which is a subset of the coordinates of $P$, is equal to the weighted sum of $(C^j,w^j)$ (by~\eqref{eq:equalSumCords}) and the definition of $\hat{C}^j$ to be the set of padded vectors in $C^j$, we obtain that
\begin{equation} \label{eq:mainSum}
\sum_{j\in [k_2]} \sum_{c \in \hat{C}^j}w(c)c = \sum_{p \in P} u(p) p.
\end{equation}

The output weighted set $(C,w)$ is then simply the union over all the padded vectors in $\hat{C}^1,\cdots,\hat{C}^{k_2}$ and their weights. Therefore,
\[
\sum_{c\in C}w(c) = \sum_{j \in [k_2]} \sum_{c \in \hat{C}^j}w(c) = \sum_{j \in [k_2]} 1 = k_2,
\]
where the second derivation is by~\eqref{eq:sumWc},
\[
\sum_{c\in C}w(c)c = \sum_{j \in [k_2]} \sum_{c \in \hat{C}^j}w(c)c =\sum_{p \in P} u(p) P,
\]
where the second equality is by~\eqref{eq:mainSum}, and
\[
|C| = \sum_{j \in [k_2]} |C^j| \leq \sum_{j \in [k_2]} (d'+1) = k_2 \cdot (d' + 1) \leq \ceil{\frac{d}{d'}}(d'+1).
\]
Furthermore, each vector in $C$ is a padded vector of $C^j \subseteq P^j$ for some $j \in [k_2]$, i.e., for every $c\in C$ there is $p\in P$ such that $c$ is a subset of the coordinates of $p$. Hence, $(C,w)$ is a $\ceil{\frac{d}{k_2}}$-Sparse Caratheodory set of $(P,u)$.

The computation time of $(C,w)$ is dominated by the loop at Line~\ref{line:dimIter}. Each iteration among the $k_2$ iterations of the loop is dominated by the call $\caraf(P^j,u^j,k_1)$ at Line~\ref{Line:callBoost}. By Theorem~\ref{theorem:fastCara}, since $P_j$ is of dimension at most $d' = \ceil{d/k_2}$ by its construction, this call takes $O\left(nd' + t(k_1,d')\cdot \frac{\log{n}}{\log{k_1/d'}}\right)$ time. The total running time is therefore $O\left(nd+\tcara(k_1,d')\cdot \frac{k_2\log n}{\log (k_1/d')}\right)$ as required.
\end{proof}

\section{Sparsified Caratheodory Matrix}

\begin{theorem} \label{theorem:cordCov2}
Let $A \in \REAL^{n\times d}$ be a matrix, and $k_1,k_2,d'$ be three integers such that $k_2 \in \br{1,\cdots,d^2}$, $d' = \ceil{\frac{d^2}{k_2}}$, and $k_1 \in \br{d'+2,\cdots,n}$. Let $S \in \REAL^{d\times d}$ be the output of a call to $\cordCova(A,k_1,k_2)$; see Algorithm~\ref{cordCovAlg}. Let $\tcara(k_1,d')$ be the time it takes to compute a Caratheodory Set for $k_1$ points in $\REAL^{d'}$, as in Theorem~\ref{caraTheorem}. Then $A^TA = S^TS$. Furthermore, $S$ is computed in $O\left(nd^2+\tcara(k_1,d')\cdot \frac{k_2\log n}{\log (k_1/d')}\right)$ time.
\end{theorem}
\begin{proof}
We consider the variables from Algorithm~\ref{cordCovAlg}.

First, note that the covariance matrix is equal to $A^TA =\sum_{i=1}^n a_ia_i^T$. 
We wish to maintain this sum using a set of only $d$ vectors.
To this end, the for loop at Line~\ref{forPcreate} computes and flattens the $d\times d$ matrix $a_ia_i^T \in \REAL^{d\times d}$ for every $i \in [n]$ into a vector $p_i \in \REAL^{t^2}$, and assigns it a weight of $1/n$.

The call $\caracord(P,u,k_1,k_2)$ at Line~\ref{compCU} returns a weighted set $(C,w)$ that is a $\ceil{d^2/k_2}$-Sparse Caratheodory set for $(P,u)$; see Theorem~\ref{theorem:cordCara}. 
Therefore,
\[
\sum_{c \in C}w(c)c = \sum_{i=1}^n u(p_i) p_i = \frac{1}{n}\sum_{i=1}^np_i,
\]
and $|C| \in O(d^2+k_2)$.
To this end, $c'$ which is computed at Line~\ref{Line:compctag} satisfies that
\[
c' = n \sum_{c \in C}w(c)c = \sum_{i=1}^n p_i.
\]
Combining that $C' \in \REAL^{d\times d}$ at Line~\ref{Line:reshapectag} is a reshaped form of $c'$, with the similar fact that $a_ia_i^T \in \REAL^{d\times d}$ is a reshaped form of $p_i$, we have that
\[
C' = \sum_{i=1}^n a_ia_i^T = A^TA.
\]
Let $C' = UDV^T$ be the thin Singular Value Decomposition of $C'$. Observe that $U = V$ since $C' = A^TA$ is a symmetric matrix.
By setting $S = \sqrt{D}V^T \in \REAL^{d\times d}$ at Line~\ref{Line:compS}, we obtain that
\[
S^TS = V\sqrt{D}\sqrt{D}V^T = VDV^T = C' = A^TA.
\]
We thus represented the sum $A^TA =\sum_{i=1}^n a_ia_i^T$ using an equivalent sum $S^TS = \sum_{i=1}^d s_is_i^T$ over $d$ vectors only, as desired.

The running time of Algorithm~\ref{cordCovAlg} is dominated by the call to Algorithm~\ref{cordAlg} at Lines~\ref{compCU} and the computation of the SVD of the matrix $C'$ at Line~\ref{Line:compS}. Since $P \subseteq \REAL^{d^2}$ and $|P| = n$, the call to Algorithm~\ref{cordCovAlg} takes $O\left(nd^2+\tcara(k_1,d')\cdot \frac{k_2\log n}{\log (k_1/d')}\right)$ time by Theorem~\ref{theorem:cordCara}, where $d' = \ceil{d^2/k_2}$. Computing the SVD of a $d\times d$ matrix takes $O(d^3)$ time. Therefore, the overall running time is $O\left(nd^2 + d^3 + \tcara(k_1,d')\cdot \frac{k_2\log n}{\log (k_1/d')}\right)$ = $O\left(nd^2+\tcara(k_1,d')\cdot \frac{k_2\log n}{\log (k_1/d')}\right)$ where the equality holds since $d \in O(n)$.
\end{proof}

\section{Corsets for SVD and PCA}

\begin{observation}[Observation~\ref{obesrv:svd}]\label{obesrv:svd2}
Let $A \in \REAL^{n\times d}$ be a matrix, $j\in\br{1,\cdots,d-1}$ be an integer, and $k \geq d^2+2$. Let $S \in \REAL^{(d^2+1)\times d}$ be the output of a call to $\cova(A,{k})$; see Algorithm~\ref{covAlg}.  Then for every matrix $Y\in \REAL^{d\times(d-j)}$ such that $Y^TY=I_{(d-j)}$, we have that
$\norm{AY}^2_F = \norm{SY}^2_F$.
\end{observation}
\begin{proof}
Combining the definition of $S$ and Theorem~\ref{theorem:cov}, we have that
\begin{equation} \label{eq:SJsvd}
A^TA = S^TS.
\end{equation}
For any matrix $B\in\REAL^{d\times d}$ let $Tr(B)$ denote its trace.
Observation~\ref{obesrv:svd} now holds as
$$\norm{AY}^2_F = Tr(Y^T(A^TA)Y) = Tr(Y^T(S^TS)Y)= \norm{SY}^2_F,$$
where the second equality is by~\eqref{eq:SJsvd}.
\end{proof}

\begin{observation}[Observation~\ref{obesrv:pca}] \label{obesrv:pca2}
Let $A = (a_1 \mid \cdots \mid a_n)^T \in \REAL^{n\times d}$ be a matrix, and $j\in\br{1,\cdots,d-1}$, $l = (d+1)^2 + 1$, and $k \geq d^2+2$ be integers. Let $(C,w)$ be the output of a call to $\pcacoreset(A,{k})$; see Algorithm~\ref{PCA-CORESET}, where $C = (c_1 \mid \cdots \mid c_l)^T \in \REAL^{l \times d}$ and $w\in \REAL^{l}$.
 Then for every matrix $Y\in \REAL^{d\times(d-j)}$ such that $Y^TY=I$, and a vector $\ell \in \REAL^{d}$ we have that
\begin{align*}
 \smi \norm{(a_i-\ell^T)Y}^2_2= \sum_{i=1}^l w_i \norm{(c_i-\ell^T)Y}^2_2,
\end{align*}
\end{observation}
\begin{proof}
Let $A' = [A\mid (1,\cdots ,1)^T] $ as defined at Line~\ref{A'line} of Algorithm~\ref{PCA-CORESET}. For every $j\in [d-k]$, let $y_j$ be the $j$th column in $Y$, and let $v_j =\ell^Ty_j$. We have that 
\begin{align}
\smi \norm{(a_i-\ell^T)Y}^2_2  \nonumber
&=\sum_{j=1}^{d-k} \smi  (a_i y_j -\ell^Ty_j)^2 \nonumber \\
&=\sum_{j=1}^{d-k} \smi  ((a_i\mid 1) (y^T_j \mid -\ell^Ty_j)^T )^2 \nonumber \\
&= \sum_{j=1}^{d-k}  \norm{A'(y^T_j \mid -v_j)^T}^2_2,\label{secondusedeq}
\end{align}
where the last equality holds by the definition of $A'$. 
 
Let $S'$ be the output of a call to $\cova(A',{k})$, and let $	S$ and $w$ be defined as in Lines~\ref{Sdefff} and~\ref{wdefff} of Algorithm~\ref{PCA-CORESET}. Hence,
\begin{align}
\sum_{j=1}^{d-j}  \norm{A'(y^T_j \mid -v_j)^T}^2_2  
&=  \sum_{j=1}^{d-j}  \norm{S'(y^T_j \mid -v_j)^T}^2_2  
= \sum_{j=1}^{d-j}  \smil {\left(s'_i (y^T_j \mid -v_j)^T \right)}^2 \label{1here}\\
&= \sum_{j=1}^{d-j}  \smil {\left((s_i^T \mid z_i)(y^T_j \mid -v_j)^T \right)}^2
=\sum_{j=1}^{d-j} \smil {(s_i^T y_j - z_iv_j) }^2\label{2here}\\
&=\sum_{j=1}^{d-j}  \smil z_i^2 {\left( (s_i^T /z_i)y_j - v_j\right) }^2
=\sum_{j=1}^{d-j} \smil w_i  {\left( c_i y_j - v_j\right) }^2\label{3here} \\
&=\sum_{j=1}^{d-j} \smil w_i \norm{ (c_i - \ell )y_j }^2_2 
= \smil w_i \norm{ (c_i - \ell^T )Y}^2_2 , \label{pcaproof}
\end{align}
where the first equality in~\eqref{1here} holds by Observation~\ref{obesrv:svd}, the first equality in~\eqref{2here} holds since $s'_i = (s_i^T\mid z_i )^T$, and the first equality in~\eqref{3here} holds by the definition of $c_i$ and $w_i$.

Combining~\eqref{pcaproof} with~\eqref{secondusedeq} proves the observation as
\begin{align*}
 \smi \norm{(a_i-\ell^T)Y}^2_2= \sum_{i=1}^l w_i \norm{(c_i-\ell^T)Y}^2_2.
\end{align*}
\end{proof}

\end{document}